\documentclass[12pt,a4paper]{article}
\usepackage[utf8]{inputenc}
\usepackage{geometry}
\usepackage{mathtools}
\usepackage{amssymb}
\usepackage{amsmath}
\usepackage{makecell}
\usepackage{amsthm}
\usepackage{multirow}
\usepackage{graphicx}
\usepackage{xcolor}
\usepackage{algorithm}
\usepackage{booktabs}
\usepackage{algpseudocode}
\usepackage{bbm}
\newcommand{\indep}{\perp\!\!\!\perp}
\renewcommand{\L}{\mathcal{L}}
\newtheorem{proposition}{Proposition}

\newtheorem{corollary}[proposition]{Corollary}
\title{Full Bayesian Reinforcement Learning via LF-IBIS}

\author{
  Stefano Masini\thanks{Department of Computer Science, University of Pisa, Italy. Email:\texttt{stefano.masini@phd.unipi.it}} \and
  Cecilia Viscardi\thanks{Department of Economics and Statistics, University of Salerno, Italy. Email: \texttt{cviscardi@unisa.it}} \and
  Michela Baccini\thanks{Department of Statistics, Computer Science, Applications ”G.Parenti”, University of Florence, Italy. Email: \texttt{michela.baccini@unifi.it}}
}

\date{} 

\begin{document}

\maketitle

\begin{abstract}
Reinforcement Learning (RL) is a sequential decision-making framework in which an agent learns optimal policies through interaction with an environment by maximizing cumulative rewards. Among RL methods, Bayesian Reinforcement Learning (BRL) addresses common practical challenges related to data scarcity by leveraging prior knowledge about the environment and sequential belief updates. However, most BRL approaches require an explicit likelihood function, which is frequently inaccessible or intractable in real-world settings.
We propose Likelihood-Free Iterated Batch Importance Sampling (LF-IBIS), a novel algorithm for BRL that updates the agent's beliefs online as new interactions become available. By combining Approximate Bayesian Computation with Iterated Batch Importance Sampling, LF-IBIS enables full Bayesian inference also in settings where the environment dynamics is not described by an explicit or tractable likelihood. The method yields approximate posterior distributions over both environment parameters and optimal policies,  providing a quantification of policy uncertainty useful for a Bayesian treatment of the exploration–exploitation trade-off.
We test the method on a simulation study in response-adaptive randomization in clinical trials, where  closed-form posteriors enable validation. 
Additional experiments address settings where the posterior has no closed form and illustrate online policy updating based on the posterior distribution of the optimal policy.

\end{abstract}

\vspace{1em}


\section{Introduction}
Reinforcement Learning (RL) is a computational framework for sequential decision-making in which an agent interacts with an environment to learn how to choose actions that maximize cumulative rewards over time \cite{sutton2018reinforcement}. At each step, the agent observes the current state, selects an action, and receives a reward or penalty along with a successor state, gradually improving its behaviour through trial and error. RL has become widely used in several domains—including robotics, healthcare and online decision systems—thanks to its ability to autonomously discover effective strategies in complex and uncertain settings.

RL methods are commonly classified into \emph{model-free} and \emph{model-based} approaches. \emph{Model-based} RL learns a model of the environment behaviour —often formalized as a Markov Decision Process (MDP)— that specifies the state-transition dynamics~\cite{Wong2022-qe}. The agent learns this model from interaction data and uses it to plan and compute its policies. In contrast, \emph{model-free} methods learn policies (or value functions) directly from observed state–action–reward sequences. In this work, we focus on the model-based formulation.

In classical RL, effective learning usually requires a large amount of data collected through many interactions with the environment \cite{sutton2018reinforcement}. In many real-world applications, however, acquiring such interaction data is expensive, limited, or risky. This exacerbates the well-known exploration–exploitation trade-off: the choice between exploring new actions and gather better information, or exploiting actions currently believed to be optimal to maximize immediate rewards \cite{sutton2018reinforcement,auer2002finite}. Bayesian Reinforcement Learning (BRL) addresses this challenge by integrating the principles of Bayesian inference with RL, enabling the incorporation of prior knowledge and a continual updating of the belief as new data become available. This leads to more data-efficient learning when interactions are limited~\cite{ghahramani2015probabilistic,dearden1999model}. 

A major obstacle of any approach grounded in a statistical formulation of the RL problem, including BRL, is the requirement for an explicit likelihood function associated with the model that describes the environment's behavior. In many real-world scenarios, such a likelihood is intractable or unavailable, for instance when the environment is modeled as a complex simulator or a black-box process~\cite{gutmann2016bayesian}.
To mitigate this limitation, we explore the use of Approximate Bayesian Computation (ABC), a class of likelihood-free inference methods that approximate the intractable likelihood through comparison between observed data and data simulated from the complex model~\cite{Sisson}. 

Although ABC has been successfully applied in various domains, its use in sequential decision-making and BRL is still limited. A few studies have explored the connection between ABC and RL. Ritto et al.~\cite{RittoBeregi} combine RL-based model selection with ABC parameter estimation in a digital-twin setting, using Thompson sampling \cite{thompson1933} (via a Beta-Bernoulli bandit) for model choice and ABC for updating the parameter posterior. Dimitrakakis and Tziortziotis~\cite{Dimitrakakis} introduce ABC into RL as a method to perform inference in environments where the transition model is unknown or the likelihood is intractable. In their framework, the agent maintains a prior distribution over environment models and updates this distribution by sampling from it, simulating trajectories, and retaining only those samples whose simulated behavior closely matches observed data, according to a predefined distance metric. This likelihood-free approach enables the construction of approximate posterior distributions over environment models, which in turn supports principled exploration through model uncertainty. Their work represents a pioneering application of ABC to RL, demonstrating its potential to handle black-box environments without requiring access to differentiable or analytically defined transition dynamics.

Our work extends this line of research in two key directions:

\textbf{(i) Online inference with progressive histories:} unlike the offline-oriented setting of~\cite{Dimitrakakis}, our approach is designed for online RL and incrementally updates posterior beliefs as new interactions with the environment occur.

\textbf{(ii) Joint posterior distributions over models and policies:} while the original ABC-RL approach focuses on computing posterior distributions only over environment models, we propagate uncertainty over candidate stochastic policies, allowing the agent to exploit the whole posterior belief. We propose adapting policies by sampling from their posterior distributions, leading to a fully Bayesian analogue of Thompson sampling, in which exploration is driven by posterior uncertainty and exploitation naturally emerges as this uncertainty decreases.

To this end, we introduce a framework for full Bayesian Reinforcement Learning (fBRL) based on a novel algorithm, \emph{Likelihood-Free Iterated Batch Importance Sampling} (LF-IBIS), which combines Sequential Monte Carlo ABC (SMC-ABC) principles \cite{Sisson} with Iterated Batch Importance Sampling \cite{Chopin} in an RL setting. A similar approach is proposed by Roy et al.~\cite{roy2025}, who introduce a generalized Bayesian framework for deep reinforcement learning. Our approach shares with Roy et al.'s framework the fully Bayesian perspective, both the methods rely on SMC for posterior approximation, and exploit multiple posterior samples, extending classical Thompson sampling. 
However, their method is based on generalized Bayesian updating via prequential scoring rules, which act as a surrogate likelihood, whereas our approach adopts a fully likelihood-free framework based on ABC and direct simulation. Moreover, their framework leverages deep generative models and expected Thompson sampling, while our method exploits full posterior information and operates in a fully online, sequential setting as new data arrive.

We complement our approach with an analysis of the asymptotic properties and validate it empirically in a Response-Adaptive Randomization problem for clinical trials~\cite{wilson2025rar}, where a conjugate Bayesian solution allows direct comparison with LF-IBIS, standard rejection ABC, and the exact posterior.

\paragraph{Structure of the paper}Section~\ref{sec:bayesian_rl} introduces RL, its statistical formulation and the extension to the full Bayesian setting. Section~\ref{sec:lf_ibis} presents the likelihood-free formulation and details our LF-IBIS algorithm, highlighting how adaptive ABC thresholding mechanism of SMC-ABC \cite{DelMoral} is integrated with online inference in an IBIS paradigm\cite{Chopin}.  
Theoretical insights are provided in Section~\ref{sec:theory}.
In Section~\ref{sec:experiments}, we test the method empirically in the context of Response-Adaptive Randomization for clinical trials. 
Section~\ref{sec:conclusion} concludes the paper with a summary and discussion. 
Supplementary material contains proofs and additional  details about algorithms and experimental settings.

\section{Full Bayesian Reinforcement Learning} \label{sec:bayesian_rl}

In this work, we aim to define a procedure for providing a full Bayesian inference methodology in RL.
To this end, we first cast RL into a statistical framework and define a likelihood function associated with the model assumed for the environment’s behavior. Then, we  exploit the statistical formulation to provide Bayesian inference for the quantities of interest. Here, the term \emph{full Bayesian} refers to an approach in which posterior uncertainty is not only represented at the level of the environment model, value functions, policy parameters, or rewards, but is propagated to the optimal policy. In contrast with Bayesian RL methods that mainly exploit posterior expected quantities, posterior summaries, or single posterior samples, our objective is to approximate the posterior distribution over the optimal policy itself.

In what follows, capital letters denote random variables, while lowercase letters their realizations. Calligraphic letters denote sets and  variables' domains. For a finite set $\mathcal{Y}$ $|\mathcal{Y}|$ denotes the cardinality and $\Delta(\mathcal{Y})$ denotes the probability simplex over $\mathcal{Y}$. $\mathbb{N}$, $\mathbb{R}$ and $\mathbb{R}^{+}$ denote the sets of natural, real and non-negative real numbers, respectively.
We use $\Pr(\cdot)$ to refer to probabilities of events, $p(\cdot)$ to generically denote probability density or mass functions,  $P$ to denote stochastic matrices. 

\subsection{Statistical Reinforcement Learning} \label{subsec:stat_rl}
In RL, an \emph{agent} learns a strategy to make optimal decisions while interacting with an \emph{environment} and receiving rewards based on its actions. The interaction unfolds as a sequence of action--state--reward triples, the so-called observed history $x=\{(a_n,s_n,r_n)\}_{n=1}^N$,  with $N$ equal to the number of agent-environment interactions.\footnote{The history can alternatively be indexed by the time at which interactions occur. Here, we assume a single interaction per time step.}  From a statistical point of view,  action \( a_n \) and state \( s_n \) can be seen as realizations of two random variables, \( A_n \in \mathcal{A}_n \) and \( S_n \in \mathcal{S}_n \), while \( r_n \), shorthand for \( r(a_n, s_n) \), is the reward resulting from a (either deterministic or stochastic) function
\[
    r_n: \mathcal{A}_n \times \mathcal{S}_n \to \mathcal{R} \subseteq \mathbb{R}.
\]

Throughout, the state space \(\mathcal{S}_n\), the action space \(\mathcal{A}_n\), and the reward space \(\mathcal{R}\) are assumed to be finite sets. 

A common assumption is that the dynamic of the system can be described by a homogeneous Markov Decision Process (MDP) \cite{Bellman1957, Wong2022-qe}, meaning that the following two assumptions hold:

\begin{itemize}
 \item[ASS.1] $S_{n+1} \indep \{a_j, s_j\}_{j=1}^{n-1} \mid (a_n,s_n)$ for each $n\in \{1,...N-1\}$. 
 
 This establishes that the sequence of states forms a first order Markov chain and that the environment's behavior is fully described through $|\mathcal{A}_n|$ transition matrices denoted by $P_n(a)$, with   entries  $\Pr(s'\mid s,a):=\Pr(S_{n+1}=s'\mid S_n=s, A_n=a)$ for $(s,s')\in \mathcal{S}_n^2$ and $a\in \mathcal{A}_n $.
  \item[ASS.2]   $\mathcal{S}_n\equiv \mathcal{S}, \mathcal{A}_n\equiv \mathcal{A}, P_n(a) \equiv P(a)$ for each $n\in\{1,...,N\}$ and $a\in \mathcal{A}$, where $\mathcal{S}$ and $\mathcal{A}$ are two finite sets,  each $P(a)$ is a $(k_S\times k_S)$-stochastic matrix, with  $k_S:=|\mathcal{S}|$.

  This establishes the homogeneity of the Markov process.
   
\end{itemize}
Note that the observed history is a realization of a multivariate discrete random variable \(X\) that takes values in  \(\mathcal{X}^N\), where \(\mathcal{X} := \mathcal{A} \times \mathcal{S} \times \mathcal{R}\). Its joint probability distribution depends on 1) $k_A:=|\mathcal{A}|$ stochastic matrices, $P(a)$, with $a\in\mathcal{A}$; 2) a stochastic policy $\pi:\mathcal{S}\to\Delta(\mathcal{A})$ 
where $\pi(a\mid s)$ is the probability of choosing action $a$ from state $s$. \footnote{A stochastic policy is a mapping $\pi: \mathcal{S} \to \Delta(\mathcal{A})$ , whereas a deterministic policy is a mapping $\pi: \mathcal{S} \to \mathcal{A}$. We restrict attention to stochastic policies.}

Given a policy $\pi(\cdot\mid \cdot)$, the \emph{value function} $V_\pi(s)$ quantifies the expected future reward that can be gained from state $s$. Assuming the transition matrices $P(a)$ as known, it can be computed recursively by the Bellman equation \cite{Bellman1957}
\begin{equation}
    \label{eq:Bellman_with_H}
V_{\pi}(s) = \sum_{a\in\mathcal{A}}\pi(a\mid s) Q_\pi(s,a) + \lambda \mathcal{H}(\pi(\cdot\mid s)) \qquad  \forall s\in\mathcal{S} 
\end{equation}
where $\mathcal{H}(\pi(\cdot\mid s))$ is the Shannon entropy of the stochastic policy 
\[
\mathcal{H}(\pi(\cdot\mid s)):=-\sum_{a\in \mathcal{A}} \pi(a\mid s)\log\pi(a\mid s)
\]
and
$Q_\pi(s,a)$ denotes state–action value function 
\[
Q_\pi(s,a):=\sum_{s'\in \mathcal{S}} \Pr(s'\mid s,a)[r(a,s')+\gamma V_\pi(s')].
\]

Here $\gamma \in [0,1]$ is the \emph{discount factor}, which determines how future rewards are weighted relative to immediate ones.  
Note that the Bellman equation in~\eqref{eq:Bellman_with_H} is not the classical RL formulation; rather, it derives from the Maximum Entropy (MaxEnt) RL framework~\cite{ziebart2008maxent}. 
Unlike standard RL---where optimizing the expected total reward often leads to deterministic policies---the MaxEnt objective includes an entropy term that explicitly promotes stochasticity. 
The parameter $\lambda\in\mathbb{R}^+$ controls the relative weight of the entropy contribution, thereby regulating the degree of randomness in the resulting policy. Larger values of $\lambda$ promote higher-entropy policies and thus greater randomness. As noted in \cite{Haarnoja2018}, the temperature parameter (typically denoted by $\alpha$ in the reinforcement learning literature) can be absorbed into the reward function through a simple rescaling by $\lambda^{-1}$. Accordingly, it does not affect the structure of the optimal policy, and we omit it in the following by fixing $\lambda = 1$ without loss of generality.

An \emph{optimal policy} $\pi^*(a \mid s)$ is any policy that allows the agent to reach the optimal value function $V^*(s) = \max_{\pi} V_\pi(s)$ for every state. Under the MaxEnt objective, several classical RL algorithms—such as Policy Iteration, Value Iteration, and Q-Iteration—can be used to compute $\pi^*$~\cite{sutton2018reinforcement}. All these methods rely on access to the transition model $P$, or on accurate estimates of either its entries or the corresponding value functions. In many cases, empirical estimates of transition probabilities can be obtained through Monte Carlo sampling~\cite{sutton2018reinforcement}.

In what follows, we cast RL within a statistical inference framework that enables learning of both the transition dynamics and the optimal policy.

\paragraph{Likelihood function in RL}
Recall that under the MDP assumptions, the dynamic of the environment is fully described by \( k_A \) matrices denoted by \( P(a) \) with \( a \in \mathcal{A} \). These are  $(k_S\times k_S)$-stochastic matrices, where each row corresponds to a state \( s \) and contains the conditional probability distribution over \(\mathcal{S}\), denoted by \(\Pr(\cdot \mid s, a)\).

Let us denote by $\mu \in \mathcal{M}$, with $\mathcal{M} \subseteq \mathbb{R}^d$ a vector of parameters of size $d$  for some $d \in \mathbb{N}$. We assume that the rows of the matrix $P(a)$ are unknown elements of a family of  conditional probability distributions parametrized by $\mu$: \[\mathcal{F}:=\{\Pr(\cdot \mid s, a, \mu ): \mu \in \mathcal{M}\}.\] 
Under this assumption, given a policy $\pi$, the random history   $X$ has a distribution represented by the following probability function:

\begin{equation}
\label{eq:jointprob}
    \Pr(X=x\mid \mu, \pi) =  \rho(s_1) \prod\limits_{n=1}^N \pi(a_n\mid s_n)\Pr(S_{n+1}=s_{n+1}\mid s_{n}, a_{n}, \mu) \qquad x\in \mathcal{X}^N,
\end{equation}
where $\rho(\cdot)$ is the initial state distribution.
This probability has key relevance when addressing inferential purposes since it is proportional to the likelihood function for the parameters $\mu$ that should be inferred to learn the environment transition matrix. 
In particular, let us define 
\begin{equation*}
\label{}
   f(x \mid \mu) \coloneqq \prod\limits_{n=1}^N \Pr(S_{n+1}=s_{n+1}\mid s_{n}, a_{n}, \mu).
\end{equation*}
Given the initial state $s_1$ and a policy $\pi(\cdot \mid \cdot)$, from Eq \eqref{eq:jointprob} follows that

\begin{equation}
\label{eq:likelihood}
    \L(\mu; x, \pi, s_1)\propto  f(x \mid \mu)
\end{equation}
where $\L(\mu; x, \pi, s_1)$ denotes the likelihood function, which will henceforth be written as $\L(\mu; x)$ for brevity.

This statistical formulation lays the foundation for inferential procedures both in a frequentist and Bayesian framework where  an agent wants to learn the behavior of an environment governed by $\mu$, based on the sequence of its states $s_1, \ldots, s_N$ in an observed history $x$. In this work, we focus on a Bayesian treatment of the RL problem, which is presented in the following section.

\subsection{Bayesian Reinforcement Learning}

Bayesian approaches to RL may represent posterior uncertainty over different quantities, including environment models, value functions, policy parameters, policy gradients, or reward functions. In the framework proposed in this work, inference is performed over the parameters $\mu$ governing the environment dynamics, in line with model-based BRL approaches such as BAMDP-based formulations \cite{duff_2002} and posterior sampling methods \cite{osband_etal2013, strens2000}. In this context, the parameters $\mu$ are treated as random variables. The agent starts with a prior belief expressed by the prior distribution $p(\mu)$ over $\mathcal{M}$. Then, it uses the likelihood function in equation \eqref{eq:likelihood} to update it and compute the posterior $p(\mu\mid x)$ through the Bayes formula:

\begin{equation}
    \label{eq:posterior_update}
    p(\mu\mid x) \propto p(\mu )\L(\mu; x).
\end{equation} 

In this work, we adopt a \emph{full Bayesian} perspective. Unlike the model-based BRL approaches the posterior is not used solely to identify a specific parameter vector—e.g., the maximum a posteriori estimate or posterior expected quantities—to evaluate the value function and determine an optimal policy, but  all the information contained in the posterior distribution in Eq \eqref{eq:posterior_update} is exploited.
In fact, the optimal policy in turn depends on the vector of parameters $\mu$ through the transition probabilities. We make this dependence explicit in  the value functions and the optimal policy.
The optimal state--action value function $Q^*_\mu(s,a)$ for a given  \(\mu\) is 
\[
\label{eq:Qstar_mu}
Q^*_\mu(s,a)
= \sum_{s'\in \mathcal{S}} \Pr(s'\mid s,a,\mu)\big[ r(a,s') + \gamma\, V^*_\mu(s') \big],
\]
where \(V^*_\mu(s)\) is the optimal (entropy-regularized) value function retrieved by solving the soft Bellman fixed-point
\[
\label{eq:Vstar_mu}
V^*_\mu(s)
= \lambda \log\!\bigg( \sum_{a \in \mathcal{A}}\exp\!\bigg(\frac{1}{\lambda} Q^*_\mu(s,a)\bigg) \bigg).
\]
 This value function corresponds to an optimal policy
\begin{equation}
\label{eq:pi_star_mu}
\pi^*_\mu(a\mid s)
= \frac{\exp\!\big( \tfrac{1}{\lambda} Q^*_\mu(s,a)\big)}
       {\sum\limits_{a'\in \mathcal{A}} \exp\!\big( \tfrac{1}{\lambda} Q^*_\mu(s,a')\big)}.
\end{equation}

The soft Bellman fixed--point equation and the corresponding Boltzmann optimal policy are standard results in maximum--entropy reinforcement learning; an explicit derivation can be found in \cite{Haarnoja2017}, building on the maximum-entropy formulation of \cite{Ziebart2010}.

We exploit their results to provide a posterior distribution for the optimal policy, rather than a point estimate. In fact, if we denote by $\pi^*_\mu$ the stochastic matrix of probabilities $\{\pi^*_\mu(a \mid s) : (a,s) \in \mathcal{A} \times \mathcal{S}\}$, then it follows from equation~\eqref{eq:pi_star_mu} that $\pi^*_\mu$ is a function of $\mu$.
To compute $\pi^*_\mu$ given a specific $\mu$ value, a standard optimal-policy computation method can be employed; in our case, we adopt Policy Iteration. Since each $\mu$ value is associated with a posterior density (or probability), the uncertainty about $\mu$ can be propagated to $\pi^*_\mu$, thus computing $p(\pi^* \mid x)$. 

Exploiting the posterior distribution of the optimal policy requires a simulation method that allows us to repeatedly sample from it. To this end, Section~\ref{sec:lf_ibis} introduces the proposed algorithm for approximating the posterior distribution of $\pi^\star$. Once this distribution has been approximated, it can be used to address the exploration--exploitation trade-off in an uncertainty-aware way. A possible decision strategy based on this posterior information is then presented in Section~\ref{sec:updating_policy}.



\subsubsection{Posterior-Based Policy Updating}
\label{subsec:posterior_policy_update}

The posterior distribution over policies provides valuable information for deciding when it is advantageous to update the current policy. 

To this end, we adopt a sampling-based decision rule grounded in the Bayesian Mean Squared Error (BMSE). Let $\{\pi^{(\ell)}\}_{\ell=1}^L$ denote a collection of policies sampled from the posterior distribution of the optimal policy. The BMSE with respect to a given policy $\pi$ is defined as follows:

\begin{equation}
\label{eq:BMSE}
BMSE(\pi) = \frac{1}{L} \sum_{\ell=1}^{L} | \pi^{(\ell)} - \pi |^2.
\end{equation}

In the implementation of an RL procedure, this quantity can be used to assess how well a given policy aligns with the posterior distribution of optimal policies. In particular, at each update step, we compute $BMSE(\pi)$ for the current policy and $BMSE(\tilde{\pi})$ for a new candidate policy $\tilde{\pi}$ drawn from the posterior distribution. The policy with the lower BMSE is then selected.

If the selected policy coincides with the current policy, we interpret this as evidence that the posterior mass is already concentrated around it, and thus further exploration of the environment may be beneficial. Conversely, if the sampled policy yields a lower BMSE, we update the current policy accordingly, exploiting the information contained in the posterior.

This procedure naturally balances exploration and exploitation: the posterior distribution encodes uncertainty over optimal policies, while the BMSE-based criterion provides a principled mechanism to decide whether to retain or update the current policy.

\section{Likelihood-free Iterated Batch Importance Sampling (LF-IBIS)}\label{sec:lf_ibis}
The posterior distribution in Eq.~\eqref{eq:posterior_update} is often not available in closed form. As a consequence, the posterior distribution over the optimal policy is not available analytically either. 
As in Bayesian analysis in general, posterior computations often require methods for getting samples from the posterior distribution such as Monte Carlo or Markov Chain Monte Carlo algorithms \cite{robert2004monte}. 
In the presented RL framework, two issues make this task particularly challenging: 

\begin{enumerate}
    \item \textbf{Intractable or unavailable likelihood.} The likelihood $\L(\mu ; x)$ depends on the transition matrices parameterized by $\mu$, whose entries may be complex or unknown functions of the parameters. In other words, the family $\mathcal{F}$ may be unknown, or it may consist of non-computable elements. This occurs because environment dynamics are often described by computational models that can be simulated but lack analytical form. 
    \item \textbf{Sequential data.} The history $x$ is observed gradually over time, and recomputing posterior distributions as new interactions agent-environment occur can be computationally prohibitive. 
\end{enumerate}

Likelihood--based  Monte Carlo methods are therefore not directly applicable. 
To address these issues, we propose a hybrid algorithm, \emph{Likelihood-free Iterated Batch Importance Sampling}, which combines Approximate Bayesian Computation (ABC) with Iterated Batch Importance Sampling (IBIS). 

\subsection{ABC in Reinforcement Learning \label{sec:ABC}}

Approximate Bayesian Computation (ABC) provides a likelihood-free approach to Bayesian inference when likelihood evaluation is infeasible or analytically intractable \cite{rubin1984bayesian,sunnaker2013abc}. Instead of computing the exact likelihood, ABC relies on a Monte Carlo approximation based on the simulation of pseudo-data, i.e. simulated histories,  from the generative model associated with the environment -- see \cite{Sisson} for details about ABC and \cite{Dimitrakakis} for a first application to RL. 
The simplest ABC sampling scheme \cite{tavare1997, pritchard1999}, i.e., the standard Rejection ABC, proceeds by sampling parameters from the prior, using these parameters to generate pseudo-data from a simulator that reproduces the generative process underlying the intractable likelihood, computing distances between simulated and observed data, and accepting parameters for which the distance is within a tolerance threshold $\epsilon>0$. The output is a sample from an approximate posterior distribution, with the quality of approximation depending on the choice of $\epsilon$.

 Let $y \in \mathcal{X}^N$ denote a simulated history, sampled from a simulator that reproduces the environment behaviour. The output of  the ABC procedure is a sample from the following approximate joint posterior distribution:

\begin{center}
    $p_\epsilon(\mu, y\mid x) = \dfrac{p(\mu) \L(\mu; y) \mathbb{I}_{A_{\epsilon,x}} (y)}{\int\limits_{\mathcal{M}}\sum\limits_{y \in A_{\epsilon,x}} p(\mu)\L(\mu ; y)   d\mu}\propto p(\mu) f(y \mid \mu) \mathbb{I}_{A_{\epsilon,x}} (y) $   .
\end{center}

Here,  $\mathbb{I}_Z(\cdot)$ is the indicator function of a set $Z$, the set $A_{\epsilon, x} = \{y \in \mathcal{X}^N : d(y, x) \leq \epsilon\}$ identifies the acceptance region,
and $d:\mathcal{X}^N \times \mathcal{X}^N \rightarrow \mathbb{R}^+$ is a distance function.

By marginalizing out $y$, we obtain the approximate posterior over the parameters $\mu$:

\begin{equation*}
    p_\epsilon(\mu\mid x) \propto p(\mu) f_\epsilon(x \mid \mu)
\end{equation*}
where:

\begin{equation}
\label{eq:approxlikelihood}
  f_\epsilon(x \mid \mu) \coloneqq \Pr\big(d(y,x)\leq \epsilon \mid \mu\big) = \sum\limits_{y\in\mathcal{X}^N} f(y \mid \mu) 
{\mathbb{I}_{A_{\epsilon,x}}(y)}.   
\end{equation}

This quantity represents the \emph{ABC likelihood}, that
can be estimated via Monte Carlo sampling. Specifically, one can draw $M$ i.i.d. pseudo-dataset $\{y_1, \dots, y_M\} \sim f(\cdot \mid \mu)$ and approximate the probability in  Eq \eqref{eq:approxlikelihood} by 
\begin{equation}
\label{eq:lMClikapprox}
f_{\epsilon,M}(x\mid \mu)\coloneqq \frac{1}{M}\sum\limits_{m=1}^M\mathbb{I}_{A_{\epsilon,x}}(y_m).
\end{equation}
Note that the probability of  drawing pseudo-data close to the observed data, $\Pr\big(d(x,y)\leq \epsilon \mid \mu\big)$, depends also on $N$. Thus, to make the algorithm more scalable and get a  good Monte Carlo approximation without requiring huge $M$, a common practice is to summarize both observed and simulated data through lower-dimensional summary statistics $\eta(\cdot)$. This implies a new definition of the acceptance region as $A_{\epsilon,\eta(x)}=\{y \in \mathcal{X}^N: d(\eta(y), \eta(x) )\leq \epsilon\}$, and the introduction of a new source of approximation. Moreover, the indicator function is often replaced by a smooth kernel function $\kappa(\cdot)$, which allows for a continuous scaling from 1 (when $x = y$) to 0 (when $d(\eta(y), \eta(x)) > \epsilon$), leading to the following posterior approximation:


\begin{equation}
\label{eq:approxposterior}
    {p}_\epsilon(\mu\mid x) \propto p(\mu)f_\epsilon(x\mid \mu) \approx p(\mu) \cdot \frac{1}{M}\sum\limits_{m=1}^M \kappa \big(d(y_m,x); \epsilon\big)
\end{equation}
Henceforth, to simplify notation, $x$ and $y$ will denote, according to the context, either the complete data or their corresponding summary statistics $\eta(x)$ and $\eta(y)$.

In Eq \eqref{eq:approxposterior} the ABC kernel $\kappa$ is typically a smoothing function (e.g., Epanechnikov or Gaussian) that uses $\epsilon$ as scale parameter and is often defined over a compact support.
Note that the standard Rejection ABC algorithm uses $M=1$ and relies on a crude Monte Carlo estimate of the ABC likelihood which, for the sake of this work, is instead a key quantity, as will become clear in the remainder of the paper.
In fact, more advanced alternative to Rejection ABC, such as the \emph{Sequential Monte Carlo Approximate Bayesian Computation} (SMC-ABC) algorithm proposed by Del Moral et al. \cite{DelMoral}, uses the Monte Carlo approximation in Eq \eqref{eq:lMClikapprox}. 
This method adapts classical SMC methodology to the ABC framework by constructing a sequence of intermediate approximate distributions \(\{p_{\epsilon_i}(\mu \mid x)\}_{i=1}^n\), each associated with one of the decreasing tolerance levels \(\epsilon_1 > \dots > \epsilon_n\). 
The algorithm iterates importance sampling steps as displayed in Alg \ref{alg:smcabc}.
At each iteration \(i\), the algorithm gets a weighted sample of size $L$, $\{ \mu_i ^{(\ell)}, \omega_i^{(\ell)}(\epsilon_i)\}_{\ell=1}^L $,  through three main steps: (i) reweighting particles from the previous iteration -- i.e. the weighted sample from ${p}_{\epsilon_{i-1}} $ -- according to the ABC likelihood that uses the current threshold \(\epsilon_i\), (ii) resampling $L$ particles using new weights, (iii) perturbing particles through a  transition kernel -- e.g., MCMC kernel of invariant distribution ${p}_{\epsilon_i}$ that uses a proposal distribution $q(\cdot,\cdot)$.

In the version proposed by Del Moral et al.~\cite{DelMoral}, the sequence \(\{\epsilon_i\}_{i=1}^n\) is not fixed in advance but instead chosen \emph{adaptively} to gradually improving the quality of the approximation while controlling the accuracy of the posterior estimate. The key idea is starting with a large tolerance \(\epsilon_1\), which allows more simulated dataset to be accepted, and then progressively reduce it to focus the inference on regions of higher posterior density.

To adapt the threshold, the algorithm monitors the \emph{Effective Sample Size} (ESS): the next threshold \(\epsilon^*\) is chosen by solving the following equation through a bisection algorithm:
\begin{equation}
\label{eq:ess}
\text{ESS}\left( \omega_i(\epsilon^*) \right) = \alpha \cdot \text{ESS} \left( \omega_{i-1}(\epsilon_{i-1}) \right),
\end{equation}
where \(\omega_i(\cdot)\) denote the corresponding (unnormalized) importance weights, and \(\alpha \in (0,1)\) is a tuning parameter (e.g., $\alpha = 0.9$) that controls how aggressively the tolerance is reduced.



\begin{algorithm}[ht]
\caption{SMC-ABC Algorithm (Del Moral et al. 2012)} \label{alg:smcabc}
{\scriptsize
\begin{algorithmic}[1]
    \State \textbf{Input:} target threshold $\epsilon$, starting threshold $\epsilon_1$, $\alpha\in(0,1)$, nr of simulations $L$, number of pseudo-data $M$, observed data $x$; proposal distribution $q(\cdot, \cdot)$
        \State Draw $\mu_1^{(\ell)} \sim p(\cdot)$ for $\ell\in \{1,...,L\}$
            \State  For each $\ell$, simulate $y_{1,m}^{(\ell)} \sim f(\cdot\mid\mu_1^{(\ell)})$ for $m\in \{1,...,M\}$
        \State Compute initial weights $\omega_1(\epsilon_1)=\{\omega_1^{(\ell)}(\epsilon_1)\}_{\ell=1}^L$ \hspace{1.3cm}where $\omega_1^{(\ell)}(\epsilon_1)=\dfrac{\sum_{m=1}^M \kappa\big(d(y_{1,m}^{(\ell)},x); \epsilon_1\big)}{M}$
\State $i=1$
    \While{$\epsilon_i>\epsilon$}
       \State Set $i=i+1$
       \State \textbf{Reweighting:}: Compute \textbf{new $\epsilon_i$} by solving with bisection method: \hspace{0.5cm}$
       \text{ESS}\big(\omega_i(\epsilon_i)\big) = \alpha \text{ESS}\big(\omega_{i-1}(\epsilon_{i-1})\big)
       $
       \State Where \hspace{0.5cm}
       $
       \omega_i^{(\ell)}(\epsilon_i) = \omega^{(\ell)}_{i-1}(\epsilon_{i-1}) \dfrac{\sum_{m=1}^M \kappa\big(d(y_{i-1,m}^{(\ell)}, x); \epsilon_i\big)} {\sum_{m=1}^M \kappa\big( d(y_{i-1,m}^{(\ell)},x); \epsilon_{i-1}\big)}
       $
       \State \textbf{Resampling:} Sample $L$ pairs from $\{ (\mu_{i-1}^{(\ell)}, y_{i-1,m}^{(\ell)})\}_{\ell=1}^L$ using their weights $\dfrac{\omega_i^{(\ell)}(\epsilon_i)}{\sum_{\ell=1}^L\omega_i^{(\ell)}(\epsilon_i)}$ as probabilities. Denote new sampled pairs, with a slight abuse of notation, by $\{(\mu_{i-1}^{(\ell)}, y_{i-1,m}^{(\ell)})\}_{\ell=1}^L$.
           \State Set new weights: \hspace{1cm}
           $
           \omega^{(\ell)}_i(\epsilon_i)=1
           $

\State \textbf{Moving:}
Draw $\mu^* \sim q(\mu_{i-1}^{(\ell)}, \cdot)$ for $\ell \in\{1,...,L\}$
                           \State For each $\ell$, simulate $y_{m}^* \sim f( \cdot \mid \mu^*)$ for $m\in\{1,...,M\}$

            \State Compute  the acceptance ratio: \hspace{0.2cm}
            $
            \text{ar}(\mu_i^{(\ell)}\rightarrow \mu^*) = \dfrac{p(\mu^*) \sum_{m=1}^M \kappa\big(d(y_{m}^*,x); \epsilon_i\big)q(\mu^*, \mu_{i}^{(\ell)})}{p(\mu_i^{(\ell)})\sum_{m=1}^M \kappa\big( d(y_{i-1,m}^{(\ell)},x); \epsilon_i\big)q(\mu_{i}^{(\ell)}, \mu^*)}
            $ for $\ell \in\{1,...,L\}$
                        \State Draw $u^{(\ell)} \sim \text{Unif}(0,1)$ for $\ell \in \{1,...,L\}$
            \If{$u^{(\ell)} < \min\{1,\text{ar}(\mu_i^{(\ell)}\rightarrow \mu^*)\}$}
                \State Set $\mu_i^{(\ell)}= \mu^*$  and $y^{(\ell)}_{i,m}=y^*_{m}
                $ for $\ell \in \{1,...,L\}$
                \Else 
                \State Set $\mu_i^{(\ell)}= \mu_{i-1}^{(\ell)}$  and $y^{(\ell)}_{i,m}=y_{i-1,m}^{(\ell)}$ for $\ell \in \{1,...,L\}$
            \EndIf
      \EndWhile
\end{algorithmic}}
\end{algorithm}


Note that in \cite{DelMoral}, the ABC kernel is specified as an indicator function, whereas Algorithm~\ref{alg:smcabc} is formulated more generally to accommodate arbitrary ABC kernels. Using multiple simulations (\( M \geq 1 \)) per particle and/or alternative kernel choices may improve the stability of the algorithm by reducing variability among the importance weights. Further details on this aspect are provided in Appendix~\ref{appendix:abc_details}.

More generally, the kernel function \( \kappa \), the summary statistics \( \eta(\cdot) \), the distance function, the metrics used to evaluate the efficiency of posterior estimates (e.g., ESS) etc, can be replaced by more sophisticated alternatives to better capture the relevant features of the data and improve inference quality. These aspects will be discussed specifically with reference to the case at hand in the following sections of the paper.

\subsection{Iterated Batch Importance Sampling (IBIS) \label{sec:ibis}}
The Iterated Batch Importance Sampling (IBIS) algorithm, introduced by Chopin \cite{Chopin}, is a Sequential Monte Carlo (SMC) method designed for static scenarios, where a single posterior distribution
is targeted.

The central idea is to build a sequential data setting to address high-dimensional problems in which standard Monte Carlo methods are impractical, due to the computational burden associated with data size and/or the lack of well-informed proposal distributions.
The algorithm iterates importance sampling steps as displayed in Alg \ref{alg:ibis}.
At iteration $i$, it draws a weighted sample 
$\{ \mu_i^{(\ell)}, \omega_i^{(\ell)} \}_{\ell=1}^L$ 
from the partial posterior 
$p(\mu \mid x_{1:i})$, 
conditioned on observations 
$x_{1:i} = \{ x_j \}_{j=1}^i$. 
The algorithm proceeds through 
a reweighting, a resampling, and a moving step. 
During the reweighting phase, the weighted sample from the previous iteration, i.e., 
the particles from $p(\mu \mid x_{1:i-1})$, is adapted to account for the new observation $x_i$, 
yielding an updated approximation of $p(\mu \mid x_{1:i})$.\footnote{For simplicity, we assume that only one observation is added per iteration, although the algorithm is generally formulated to handle batches larger than one.} 
Starting from the prior, the algorithm iteratively refines the posterior by using the approximate distribution from the previous iteration 
as the proposal for the next iteration. Therefore, at the $i$-th iteration, the unnormalized importance weight for the $\ell$-th particle is
\begin{equation}
\label{eq:ibisreweight}
   \omega^{(\ell)}_{i} = \omega^{(\ell)}_{i-1} \dfrac{p(\mu^{(\ell)}_{i-1}\mid x_{1:i})}{p(\mu^{(\ell)}_{i-1}\mid x_{1:i-1})}= \omega^{(\ell)}_{i-1}\dfrac{p(\mu^{(\ell)}_{i-1})f(x_{1:i}\mid \mu^{(\ell)}_{i-1})}{p(\mu^{(\ell)}_{i-1})f(x_{1:i-1}\mid \mu^{(\ell)}_{i-1})} 
   =\omega^{(\ell)}_{i-1} \cdot f(x_{i} \mid x_{1:i-1}, \mu_{i-1}^{(\ell)})
\end{equation}
where $f(x_{i} \mid x_{1:i-1}, \mu_{i-1}^{(\ell)})$ is the predictive likelihood and serves as reweighting factor  to account for the new data point $x_{i}$.

\begin{algorithm}[H]
\caption{IBIS Algorithm (Chopin, 2002)} \label{alg:ibis}
{\scriptsize
\begin{algorithmic}[1] 
\State{\textbf{Input}} Prior \( p(\mu) \), number of particles \( L \), observed data $x$, proposal distribution $q(\cdot, \cdot)$
\State Draw $\mu_1^{(\ell)} \sim p(\mu)$ for $\ell \in \{1,...,L\}$
\State Compute initial weights $\omega_1=\{\omega_1^{(\ell)}\}_{\ell=1}^L$  \hspace{1.3cm}where $\omega_1^{(\ell)}=1$  for $\ell \in \{1,...,L\}$
\For{each time step \( i = 2,\dots,n \)}
    \State \textbf{Reweighting:}
    $\omega_{i}^{(\ell)} = \omega_{i-1}^{(\ell)} \cdot f(x_i \mid x_{1:i-1}, \mu_{i-1}^{(\ell)})$ \hspace{1.3cm} for $\ell \in \{1,...,L\}$
    \State \textbf{Resampling:} Sample \( L \) particles $\{\mu_{i-1}^{(\ell)}, \omega_i^{(\ell)}\}_{\ell=1}^L$ using weights $\dfrac{\omega_i^{(\ell)}}{\sum_{\ell=1}^L\omega_i^{(\ell)}}$ as probabilities

            \State Set new weights: \hspace{0.5cm}
            $
            \omega^{(\ell)}_i=1
            $
            
    \State \textbf{Moving:} Draw $\mu^* \sim q(\mu_{i-1}^{(\ell)}, \cdot)$ for $\ell \in \{1,...,L\}$ 


    \State Compute  the acceptance ratio: \hspace{0.2cm}
            $
            \text{ar}(\mu_{i-1}^{(\ell)}\rightarrow \mu^*) = \dfrac{p(\mu^*) q(\mu^*, \mu_{i-1}^{(\ell)})}{p(\mu_{i-1}^{(\ell)})q(\mu_{i-1}^{(\ell)}, \mu^*)}
            $ for $\ell \in\{1,...,L\}$
    \State Draw $u^{(\ell)} \sim \text{Unif}(0,1)$ for $\ell \in\{1,...,L\}$
            \If{$u^{(\ell)} < \min\{1,\text{ar}(\mu_{i-1}^{(\ell)}\rightarrow \mu^*)\}$}
                \State Set $\mu_i^{(\ell)}= \mu^*$ for $\ell \in\{1,...,L\}$
            \Else 
                \State Set $\mu_i^{(\ell)}= \mu_{i-1}^{(\ell)}$ for $\ell \in\{1,...,L\}$
            \EndIf
\EndFor
\end{algorithmic}}
\end{algorithm}

This sampling approach is particularly well adapted to situations in which data become gradually available rather than being accessible from the beginning, like in RL.
However, in our framework, pointwise evaluations of the likelihood function are infeasible; therefore, it cannot be directly applied. The next section presents a new algorithm that extends its applicability to reinforcement learning. 

\subsection{LF-IBIS}
In RL, at each time step, the agent observes the outcome of its action and uses this information to improve its understanding of the environment. This means that we need an algorithm that updates the posterior for $\mu$, adapting a sample from $p(\mu \mid x_{1:i})$ to a sample from $p(\mu \mid x_{1:i+1})$, as new data $x_{i+1}$ become available.
To this end, our proposal is an algorithm that alternates two key phases:
\begin{itemize}
    \item \textbf{IBIS steps}, where the observed history is extended by including the new interaction;
    \item \textbf{SMC-ABC steps}, where the discrepancy threshold $\epsilon_i$ is adaptively reduced.
\end{itemize}
 The proposed algorithm is presented in Alg.~\ref{alg:lfibis}, and its main computational flow is summarized in the flowchart shown in Fig.~\ref{fig:lfibis-flowchart}, reported in Appendix~\ref{appendix:LF-IBIS flowchart}. Note that, while the resampling and the moving step are coherent with the original algorithms described in Sect \ref{sec:ABC} and Sect \ref{sec:ibis}, both phases require specifically tailored reweighting steps  to work in a setting where the likelihood function is not available in closed form and inference is performed in a sequential data context.
 \subsubsection{The IBIS step}
In the standard IBIS algorithm, the reweighting factor turns out to be equal to the predictive likelihood -- see Eq \eqref{eq:ibisreweight}.  This quantity, in the RL framework, has not a closed form. For this reason we propose the following \emph{ABC-like estimate} for the reweighting factor based on ABC approximations of the two quantities, $f(x_{1:i-1}\mid \mu^{(\ell)}_{i})$ and $f(x_{1:i}\mid \mu^{(\ell)}_{i})$ as in Eq \eqref{eq:approxposterior}:
$$f(x_{i} \mid x_{1:i-1}, \mu_{i}^{(\ell)}) \approx
\tilde{f}_{\epsilon, M}(x_{i} \mid x_{1:i-1}, \mu_{i}^{(\ell)}) =
\dfrac{\sum\limits_{m=1}^{M} \kappa\left( d(y_{1:i,m}^{(\ell)}, x_{1:i}); \epsilon \right)}{\sum\limits_{m=1}^{M} \kappa\left( d(y_{1:i-1,m}^{(\ell)}, x_{1:i-1}); \epsilon \right)}.$$



A crucial detail lies in the simulation process during the IBIS step -- see step 13 in Alg \ref{alg:lfibis}. We generate $M$ pseudo-histories for each   $\mu_i^{(\ell)}$ by appending one new agent/environment interaction to the already existing simulated history. Formally, this process is denoted as:
\[
\{y^{(\ell)}_{1:i,m} \sim f(\cdot \mid \mu_i^{(\ell)},y^{(\ell)}_{1:i-1,m})\}_{m=1}^M,
\]
where each $y^{(\ell)}_{1:i,m}$ represents a full trajectory of length $i$, and where $y^{(\ell)}_{1:i-1,m}$ denotes the simulated history at the previous step from that specific particle parameter $\mu_i^{(\ell)}$.

\subsubsection{SMC-ABC step}
Despite appearing as a standard iteration of Algorithm \ref{alg:smcabc}, the SMC-ABC steps present some specific features. In particular, the size of the observed data increases throughout the algorithm, which reduces acceptance probabilities and may lead to sample degeneracy \cite{liu1998sequential, Chopin}. This behavior may have two main drawbacks: (1) the effective sample size can decrease even for a fixed $\epsilon$ value; (2) the Metropolis–Hastings acceptance ratio in the moving step may become unstable, even for large values of $M$. For these reasons, it is necessary to define summary statistics, a distance function, a kernel function, and an $\epsilon$-reduction criterion specifically tailored to this framework.
\paragraph{Acceptance-rejection criteria} We will use two alternative summary statistics $\eta(\cdot)$ with different distance functions:

\begin{itemize}
    \item \textbf{Observation-based criterion}: 
    We define an empirical distribution $P^x$, represented as a third-order tensor of dimensions $k_S \times k_S \times k_A$, where $k_S := |\mathcal{S}|$ and $k_A := |\mathcal{A}|$. The entry $P^x_{j,k,z}$ gives the following empirical probability (relative frequency) of transition, estimated from the observed history $x_{1:i} = \{(a_1,s_1,r_1),\ldots,(a_i,s_i,r_i)\}$: 
    \[
    P^x_{j,k,z}
    \;=\;
    \frac{
    \sum_{t=1}^{i-1} 
    \mathbb{I}\big\{a_t=z, s_t=j, s_{t+1}= k \big\}
    }{
    i-1
    },
    \]
    
    where $\mathbb{I}\{\cdot\}$ denotes the indicator function. Specifically, the entry in row $j$ and column $k$ corresponds to the empirical frequency of transitions from state $j$ to state $k$ when action $z$ was taken. If a state--action pair is not observed in the history, the corresponding row is conventionally set to zero.
    
    Similarly, for each simulated history $y^{(\ell)}_{1:i,m}$, we define a corresponding empirical distribution $P^y$ whose entries $P^y_{j,k,z}$ are computed in the same way as $P^x_{j,k,z}$.
    
    We then evaluate the Hellinger distance between these two empirical distributions as
    
    \begin{equation}
    \label{eq:hellinger_distance}
    d\big(y^{(\ell)}_{1:i,m}, x_{1:i}\big)
    = \sqrt{
    \frac{1}{2} \sum_{z \in \mathcal{A}} \sum_{j \in \mathcal{S}} \sum_{k \in \mathcal{S}} 
    \left( \sqrt{P^y_{j,k,z}} - \sqrt{P_{j,k,z}^x} \right)^2
    }.
    \end{equation}

    \item \textbf{Utility-based criterion}: Let $U$ denote the utility function, defined as: $U = \sum_{t=1}^{\infty} \gamma^{t-1} r_t$, where $\gamma \in [0, 1]$ is the discount factor commonly used in RL, and $r_t$ the reward received at time $t$. Given the observed history $x_{1:i}$ and a simulated history $y$ at step $i$, we compute the corresponding utilities:
    
    \begin{center}
    $U_{x_{1:i}} = \sum_{t=1}^{i} \gamma^{t-1} r_t^x, \quad U_{y_{1:i}} = \sum_{t=1}^{i} \gamma^{t-1} r^y_t$    
    \end{center}
    
    where with $r^x_t$ and $r^y_t$ we indicate the reward at time $t$ in the observed and simulated histories respectively. We then define the discrepancy measure as the Euclidean distance between the two utilities:
    
    \begin{equation}
    \label{eq:Euclidean_distance}
        d\big(y^{(\ell)}_{1:i,m}, x_{1:i}\big) = \left| U_{x_{1:i}} - U_{y^{(\ell)}_{1:i,m}} \right|
    \end{equation}

    This distance captures the difference in cumulative discounted rewards between the observed and simulated histories, emphasizing similarity in overall performance rather than specific sequence details.

\end{itemize}

\paragraph{Kernel function}
We consider the following hybrid ABC kernel:
        \begin{equation}
        \label{eq:ourkernel}
            \kappa\big(d(y_{1:i,m}^{(\ell)},x_{1:i}); \epsilon_i\big) = 
            \begin{cases}
                  1 & \text{if } d(y_{1:i,m}^{(\ell)},x_{1:i}) \leq \epsilon_i  \\
                  e^{-\dfrac{d(y_{1:i,m}^{(\ell)},x_{1:i})}{\epsilon_i^2}} & \text{if } d(y_{1:i,m}^{(\ell)},x_{1:i}) > \epsilon_i
            \end{cases}
        \end{equation} 
        This kernel combines a uniform behavior for distances below the threshold $\epsilon_i$ with an exponential decay for distances above it. It enables sharp acceptance for accurate simulations while gradually penalizing less accurate ones. By avoiding a hard cutoff, it assigns non-zero weights even to simulations that slightly exceed the threshold. Since $M$ pseudo-observations are simulated per particle, this kernel allows each of them to contribute to the approximation of the likelihood function in the particle’s weight, even if their distances exceed $\epsilon_i$. Although this contribution is exponentially small, it helps stabilize the acceptance ratio computed in the \emph{moving} step of Alg \ref{alg:lfibis}. Moreover, based on arguments similar to those introduced in \cite{Viscardi2021}, we conjecture that it mitigates particle degeneracy issues --typically affecting algorithms based on importance sampling-- which, in this setting, are further exacerbated by the fact that the size of the observed data increases throughout the algorithm, thereby raising the probability of generating pseudo-data lying at a distance larger than $\epsilon$ from the observed history. However, the use of a kernel defined on a non-compact support introduces a bias; nonetheless, this bias tends to zero as $\epsilon \to 0$ see Proposition \ref{prop:bias}.

\paragraph{Adapting the threshold}
To adapt the threshold $\epsilon$, as an alternative to the ESS criterion described in Sect.~\ref{sec:ABC}, we consider a strategy based on the \textit{number of unique particles} surviving after the resampling step performed with $\epsilon^*$ \cite{Bernton2019}.
Specifically, let $\{\mu_i ^{(\ell)}\}_{\ell=1}^L$ be $L$ particles resampled at the $i$-th iteration according to weights $\omega_i^{(\ell)}(\epsilon^*)$.  We define: $$\text{UP}(\omega_i(\epsilon^*)):= \{
 \mu \in \mathcal{M}: \exists \mu_i^{(\ell)} = \mu\}.$$ The number of unique particles corresponds to $\left|\text{UP}(\omega_i(\epsilon^*))\right|$.
The next threshold $\epsilon^*$ is then computed by solving the following equation via the bisection method: 

\begin{equation}
\left|\text{UP}\left( \omega_i(\epsilon^*) \right) \right|= \alpha \cdot \left|\text{UP} \left( \omega_i(\epsilon_i) \right)\right|.
\end{equation}
Unlike the ESS, whose evaluation depends on the variability among the particle weights \cite{elvira2022rethinking}, the number of unique particles provides a measure of the structural diversity of the population after the resampling step \cite{Bernton2019}.


\begin{algorithm}[ht]
\caption[LF-IBIS]{LF-IBIS. The corresponding flowchart is reported in Appendix~\ref{appendix:LF-IBIS flowchart}, Fig.~\ref{fig:lfibis-flowchart}.}
 \label{alg:lfibis}
{\scriptsize
\begin{algorithmic}[1]
    \State \textbf{Input:} target threshold $\epsilon$, starting threshold $\epsilon_1$, $\alpha\in(0,1)$, number of simulations $L$, number of pseudo-data $M$, final length of the history $n$, initial observed data $x_{1}$; proposal distribution $q(\cdot, \cdot)$; $\epsilon$ adaptation criterion (ESS or UP)
        \State Draw $\mu_1^{(\ell)} \sim p(\cdot)$ for $\ell \in \{1,...,L\}$
            \State  For each $\ell$, simulate $y_{1,m}^{(\ell)} \sim f(\cdot\mid\mu_1^{(\ell)})$ for $m\in \{1,...,M\}$
        \State Compute initial weights $\omega_1(\epsilon_1)=\{\omega_1^{(\ell)}(\epsilon_1)\}_{\ell=1}^L$ \hspace{1.3cm}where $\omega_1^{(\ell)}(\epsilon_1)=\dfrac{\sum_{m=1}^M \kappa\big(d(y_{1,m}^{(\ell)},x_{1}); \epsilon_1\big)}{M}$
\State $i=1$
    \While{$i\leq n$  or $\epsilon_i>\epsilon$}
        \If {\textbf{no new agent-env interactions are introduced then perfom a SMC-ABC step}}
            \State \textbf{Reweighting:} Compute \textbf{new $\epsilon^*$} by solving with bisection method: 
            
            \medskip
            \hspace{1.5cm}$\text{ESS}\big(\omega_i(\epsilon^*)\big) = \alpha \text{ESS}\big(\omega_i(\epsilon_i)\big)$ 
            \hspace{0.5cm}
            or
            \hspace{0.5cm} $\left|\text{UP}\left( \omega_i(\epsilon^*) \right) \right|= \alpha \left|\text{UP} \left( \omega_i(\epsilon_i) \right)\right|$
            
            \medskip
            \State Where \hspace{0.5cm}
            $
            \omega_i^{(\ell)}(\epsilon^*) = \omega^{(\ell)}_i(\epsilon_i) \dfrac{\sum_{m=1}^M \kappa\big(d(y_{1:i,m}^{(\ell)}, x_{1:i}); \epsilon^*\big)} {\sum_{m=1}^M \kappa\big( d(y_{1:i,m}^{(\ell)},x_{1:i}); \epsilon_{i}\big)}
            $
            \State Set $\epsilon_i=\epsilon^*$
        \Else {\textbf{ if new data is introduced}} \textbf{perfom a IBIS step}
        
                \State Set $i=i+1$ and the new $\epsilon_i$ equal to the previous one
                \State For each $\ell$, simulate a new state and append it to $y^{(\ell)}_{1:i-1,m}$, for $m\in\{1,...,M\}$ :  
                
                \medskip
                \hspace{4cm} $y^{(\ell)}_{1:i,m} \sim f(\cdot \mid \mu_{i-1}^{(\ell)}, y^{(\ell)}_{1:i-1,m})$
                \medskip
                
            \State Set $\{ \mu_i^{(\ell)}\}_{\ell=1}^L = \{ \mu_{i-1}^{(\ell)}\}_{\ell=1}^L$
            \State Compute new weights:
            $
            \omega_i^{(\ell)}(\epsilon_i) = \omega^{(\ell)}_{i-1}(\epsilon_i)\dfrac{\sum_{m=1}^M \kappa\big(d(y_{1:i,m}^{(\ell)},x_{1:i}); \epsilon_i\big)} {\sum_{m=1}^M \kappa\big(d(y_{1:i-1,m}^{(\ell)}, x_{1:i-1}); \epsilon_i\big)}
            $

        \EndIf

            \State \textbf{Resampling:} Sample $L$ pairs from $\{ (\mu_i^{(\ell)}, y_{1:i,m}^{(\ell)})\}_{\ell=1}^L$ using $\dfrac{\omega_i^{(\ell)}(\epsilon_i)}{\sum_{\ell=1}^L\omega_i^{(\ell)}(\epsilon_i)}$ as probabilities. Denote new sampled pairs, with a slight abuse of notation, by $\{(\mu_{i}^{(\ell)}, y_{1:i,m}^{(\ell)})\}_{\ell=1}^L$.
            \State Set new weights: \hspace{1cm}
            $
            \omega^{(\ell)}_i(\epsilon_i)=1
            $
\State \textbf{Moving:}
Draw $\mu^* \sim q(\mu_{i}^{(\ell)}, \cdot)$ for $\ell \in \{1,...,L\}$
                           \State For each $\ell$, simulate $y_{1:i,m}^* \sim f( \cdot \mid \mu^*)$ for $m\in\{1,...,M\}$

            \State Compute  the acceptance ratio: \hspace{0.2cm}
            $
            \text{ar}(\mu_i^{(\ell)}\rightarrow \mu^*) = \dfrac{p(\mu^*) \sum_{m=1}^M \kappa\big(d(y_{1:i,m}^*,x_{1:i}); \epsilon_i\big)q(\mu^*, \mu_{i}^{(\ell)})}{p(\mu_i^{(\ell)})\sum_{m=1}^M \kappa\big( d(y_{1:i,m}^{(\ell)},x_{1:i}); \epsilon_i\big)q(\mu_{i}^{(\ell)}, \mu^*)}
            $ for $\ell \in \{1,...,L\}$
                        \State Draw $u^{(\ell)} \sim \text{Unif}(0,1)$ for $\ell \in \{1,...,L\}$
            \If{$u^{(\ell)} < \min\{1,\text{ar}(\mu_i^{(\ell)}\rightarrow \mu^*)\}$}
                \State Set $\mu_i^{(\ell)}= \mu^*$  and $y^{(\ell)}_{1:i,m}=y^*_{1:i,m}
                $  for $\ell \in \{1,...,L\}$
            \Else 
                \State Set $\mu_i^{(\ell)}= \mu_{i-1}^{(\ell)}$  and $y^{(\ell)}_{1:i,m}=y_{1:i-1,m}^{(\ell)}$ for $\ell \in \{1,...,L\}$
            
            \EndIf
        \State \textbf{For each $\mu_i^{(\ell)}$, compute the correspondent optimal policy $\pi^*_{\mu_i^\ell}$ by solving Eq.~\eqref{eq:pi_star_mu}}
      \EndWhile     
\end{algorithmic}}
\end{algorithm}

\subsubsection{Output and stopping criterion}
LF-IBIS terminates when the maximum number of iterations is reached (i.e., $i = n$), or when $\epsilon_i < \epsilon$. 

At each iteration $i$, the algorithm produces a set of $L$ particles $\{\mu_i^{(\ell)}\}_{\ell=1}^L$, which approximate the posterior distribution over the model parameters. For each particle $\mu_i^{(\ell)}$, the corresponding optimal policy $\pi^*_{\mu_i^{(\ell)}}$ is computed by solving Eq.~\eqref{eq:pi_star_mu} via policy iteration, using the transition probabilities $\Pr(s' \mid s, a, \mu_i^{(\ell)})$. 

As a result, at each iteration we obtain a sample $\{\pi^*_{\mu_i^{(\ell)}}\}_{\ell=1}^L$ from an approximation of the posterior distribution over optimal policies.

This sequential availability of posterior samples over policies enables the application of the decision criterion introduced in Subsection~\ref{subsec:posterior_policy_update}. In particular, these samples are used to evaluate the Bayesian Mean Squared Error (BMSE) and determine whether the current policy should be updated or retained, thereby supporting an adaptive exploration--exploitation strategy.

The reliability of these estimates depends on the convergence of the particle system generated by Algorithm~\ref{alg:lfibis} to the true posterior distribution. As discussed in the next section, the resulting estimators are expected to be biased but consistent.

\subsubsection{Asymptotic properties} \label{sec:theory}
Algorithm \ref{alg:lfibis} builds upon the IBIS algorithm proposed by Chopin \cite{Chopin} and the SMC-ABC algorithm by Del Moral et al. \cite{DelMoral}.  As regards SMC-ABC, the algorithm is in turn grounded in the likelihood-based SMC algorithm in Del Moral et al.\cite{DelMoralDoucetJasra2006}. While in \cite{DelMoral} they only discuss convergence results in their adaptive framework -- i.e. when the tolerance threshold is adaptively determined--   formal proofs can be found in Beskos et al. \cite{Beskos2016}.  Regarding IBIS, since it is effectively a particle filter algorithm, estimators'  consistency  and their asymptotic normality follow from theory about IS and particle filter algorithms --see \cite{CrisanDoucet2000, DoucetFreitasGordon2001,Chopin}. Morover, in \cite{Chopin} Theorem 1 shows that, under regularity conditions (see Appendix to \cite{Chopin} ),  relative precision of batch importance sampling remains asymptotically stable as the number of observations increases, provided the proportion of new data is constant, which ensures that the number of particles required to maintain accuracy does not need to grow with the total number of observations.
However, in our framework, consistency does not directly follow from these results since the reweighting factor is replaced by ABC-like estimates.


Given an observed history $x\in\mathcal{X}^i$, for same $i\geq 1$,  an i.i.d. sample of simulated histories $y_1,\ldots,y_M$ from $f(\cdot \mid \mu)$, and a predefined ABC acceptance region $A_{\epsilon,x}$, we can compute the following two alternative approximate likelihoods:
\begin{itemize}
    \item $f_{\epsilon,M}(x\mid \mu)$ : the standard ABC estimator defined as in Eq.~\eqref{eq:lMClikapprox}
    \item $\tilde{f}_{\epsilon,M}(x\mid \mu)
:= \frac{1}{M}\sum_{m=1}^M 
\kappa\!\left(d(y_m,x);\, \epsilon\right)$, our ABC estimator based on the kernel 
$\kappa(\cdot;\cdot)$ introduced in Eq.~\eqref{eq:ourkernel}. 
\end{itemize}
In what follows Proposition \ref{prop:bias} states that our approximate likelihood $\tilde{f}_{\epsilon,M}(x\mid \mu)
$  is  a biased estimate for $f_{\epsilon,M}(x\mid \mu)$  and provides an upper bound for that bias. Proposition \ref{prop:convergence} guarantees the convergence of the corresponding approximate IBIS reweighting factor to the exact one. Moreover, empirical evidence of convergence is provided through the experiments in Sect \ref{sec:experiment}.
\begin{proposition}
\label{prop:bias}
Let $\beta(\mu;\epsilon, M) :=\big\vert\tilde{f}_{\epsilon,M}(x\mid \mu) - f_{\epsilon,M}(x\mid \mu)\big\vert$ 
be the bias induced by the proposed kernel function. For each $\mu \in \mathcal{M}$, the bias is bounded from above by the following quantity that depends on $\mu$ and $\epsilon$:
\[
\beta(\mu;\epsilon, M)
\;\le\;
L(\mu;\epsilon, M)
= \frac{1}{M}\sum_{m=1}^M 
\exp\!\left[-\frac{d(x,y_m)}{\epsilon^2}\right].
\]
\end{proposition}
Proof of Proposition~\ref{prop:bias} is provided in the Supplementary Material (Appendix \ref{prop:bias_proof}).

Note that this upper bound corresponds to the extreme case in which all histories simulated under $\mu$ are far from the observed history and would have been rejected under the standard uniform kernel. From Proposition \ref{prop:bias}, it also follows that as $\epsilon$ goes to zero, $\beta(\mu; \epsilon, M) \longrightarrow 0$.


\begin{proposition}
\label{prop:convergence}
Let $x_{1:i}$  and $x_{1: i-1}$ denote the observed histories  up to iteration $i$ and $i-1$, respectively.
For a given parameter vector $\mu$, let
$\{y_{1:i-1,m}\}_{m=1}^M$
be i.i.d. simulated histories generated from
$f(\cdot\mid\mu)$ and let each
$y_{i,m}$ be independently sampled from
$f(\cdot\mid y_{1:i-1,m},\mu)$.

Define the approximate predictive likelihood used in the IBIS reweighting step as
\[
\tilde f_{\epsilon,M}(x_i\mid x_{1:i-1},\mu)
=
\frac{\tilde f_{\epsilon,M}(x_{1:i}\mid\mu)}
{\tilde f_{\epsilon,M}(x_{1:i-1}\mid\mu)}
=
\frac{\sum_{m=1}^M
\kappa(d(y_{1:i,m},x_{1:i});\epsilon)}
{\sum_{m=1}^M
\kappa(d(y_{1:i-1,m},x_{1:i-1});\epsilon)},
\]
where $\kappa(\cdot,\cdot)$ is defined in
Eq.~\eqref{eq:ourkernel} and
$d(\cdot,\cdot)$ denotes the Euclidean distance.

Then,
\[
\lim_{M\to\infty}
\lim_{\epsilon\to0}
\tilde f_{\epsilon,M}(x_i\mid x_{1:i-1},\mu)
=
f(x_i\mid x_{1:i-1},\mu)
\qquad\text{a.s.}
\]

provided that
$f(x_{1:i-1}\mid\mu)>0$.
\end{proposition}

\begin{proof}
 Consider the numerator $\tilde{f}_{\epsilon, M}(x_{1:i} \mid \mu)$. We first take the limit $\epsilon \rightarrow 0$. From Proposition \ref{prop:bias}, it follows that 
\begin{align*}
    \lim_{\epsilon\rightarrow 0}\tilde{f}_{\epsilon, M}(x_{1:i} \mid \mu)&=\lim_{\epsilon\rightarrow 0}\Bigg[f_{\epsilon, M}(x_{1:i} \mid \mu) + \beta(\mu;\epsilon, M)\Bigg]\\
    &=\lim_{\epsilon\rightarrow 0}f_{\epsilon, M}(x_{1:i} \mid \mu)\\
    &=\lim_{\epsilon\rightarrow 0}\frac{1}{M}\sum \limits_{m=1}^M\mathbb{I}_{A_{\epsilon,x_{1:i}}}(y_{1:i,m})\\
    &=\frac{1}{M}\sum \limits_{m=1}^M\mathbb{I}_{x_{1:i}}(y_{1:i,m}) 
\end{align*}
where $\mathbb{I}_{x_{1:i}}(y_{1:i,m})$ denotes the indicator function that takes the value 1 when $y_{1:i,m}=x_{1:i}$ and 0 otherwise. This function is a Bernoulli random variable whose expected value is the probability $f(x_{1:i}\mid\mu)$.
Hence, from the Strong Law of Large Numbers it follows that
\begin{equation*}
     \lim_{M\rightarrow\infty}\lim_{\epsilon\rightarrow 0}\tilde{f}_{\epsilon, M}(x_{1:i} \mid \mu)=\lim_{M\rightarrow\infty}\frac{1}{M}\sum \limits_{m=1}^M\mathbb{I}_{x_{1:i}}(y_{1:i,m})
     = f(x_{1:i}\mid\mu) \qquad a.s.
\end{equation*}
The same arguments apply to the denominator:
\begin{equation*}
   \lim_{M\rightarrow   \infty} \lim_{\epsilon\rightarrow 0}  \tilde{f}_{\epsilon, M}(x_{1:i-1} \mid \mu) = f(x_{1:i-1}\mid\mu) \qquad a.s
\end{equation*}
Since in Algorithm~\ref{alg:lfibis}, the  condition  $f(x_{1:i-1}\mid\mu)>0$ is always satisfied because the observed history has strictly positive probability under the model, we have that
\begin{equation*}
  \lim_{M\rightarrow\infty}\lim_{\epsilon\rightarrow 0}\dfrac{\tilde{f}_{\epsilon, M}(x_{1:i} \mid \mu)}{\tilde{f}_{\epsilon, M}(x_{1:i-1} \mid \mu)}= \dfrac{f(x_{1:i}\mid\mu)}{f(x_{1:i-1}\mid\mu)}= f(x_{i} \mid x_{1:i-1}, \mu) \qquad a.s.
\end{equation*}

\end{proof}

  \begin{corollary}
\label{cor:weights}
Under the assumptions of Proposition \ref{prop:convergence}, let
\[
\tilde{\omega}_i^{(\ell)}
=
\tilde{\omega}_{i-1}^{(\ell)}
\,\tilde{f}_{\epsilon,M}
(x_i \mid x_{1:i-1},\mu^{(\ell)}) \;\; \text{and} \;\; \omega_i^{(\ell)}
=
\omega_{i-1}^{(\ell)}
\,f(x_i \mid x_{1:i-1},\mu^{(\ell)})
\]
denote the approximate importance weight computed in Algorithm~3, and 
 corresponding exact IBIS importance weight, respectively. Then, for every particle
$\mu^{(\ell)}$,
\[
\lim_{M\rightarrow\infty}
\lim_{\epsilon\rightarrow0}
\tilde{\omega}_i^{(\ell)}
=
\omega_i^{(\ell)}
\qquad \text{a.s.}
\]

\end{corollary}

\begin{proof}
The result follows immediately from Proposition~\ref{prop:convergence} by substituting the approximate predictive likelihood with its limiting value in the recursive weight update.
\end{proof}
 Corollary~\ref{cor:weights} provides a theoretical justification for the reweighting step employed by Algorithm~\ref{alg:lfibis}. In the asymptotic regime $M\to \infty$ and $\epsilon \to 0$, the proposed likelihood-free reweighting scheme recovers the exact IBIS reweighting mechanism.
 
\section{Experiments} \label{sec:experiments}
We evaluate our algorithm in a controlled setting --where exact Bayesian inference is tractable-- and in scenarios where no closed-form posterior is available. This allows us to assess three complementary aspects. First, by comparing LF-IBIS with the exact posterior, we verify that the approximation induced by the ABC tolerance~$\epsilon$ remains negligible and that the method provides a good approximation to the true posterior distribution. Second, we compare LF-IBIS against acceptance--rejection ABC (AR-ABC), which represents the state of the art in likelihood-free inference in RL, as used in \cite{Dimitrakakis}, and is executed in an offline manner using the complete observed history available at the final time. Third, for both summary statistics and distance metric considered—the Hellinger distance on observation-based statistics and the Euclidean distance on utility-based statistics—we report results for LF-IBIS under the two adaptive thresholding criteria: the one based on the ESS and the one based on the number of UP.

The experiments are conducted in an online setting in which data arrive sequentially through real-time interactions with the environment, while keeping the policy fixed throughout the entire execution. In a separate experiment, we show how to exploit the full posterior distribution of the optimal policy for adaptation, demonstrating that LF-IBIS naturally supports the exploration–exploitation trade-off as posterior uncertainty is progressively reduced.

Once no further data are available, the algorithm continues to decrease the tolerance~$\epsilon$ until the target value is reached.

For comparability across methods, the observed history is kept identical across all experiments.

\paragraph{Tuning parameters}

The parameter $M$ is fixed to $M = 50$ across all experiments, while the number of particles $L$ is chosen according to the complexity of the problem and, in particular, to the informativeness of the summary statistics: larger values of $L$ are used when the statistics are less informative, in order to ensure a more accurate approximation of the posterior distribution.

Finally, the parameter $\alpha$ is calibrated empirically in a preliminary phase, once $M$ and $L$ are fixed. In particular, $\alpha$ is selected so as to control the proportion of particles discarded at each iteration, ensuring that it is appropriate for the complexity of the experiment and for the chosen values of $M$ and $L$.
Table~\ref{tab:tuning_parameters} reports the tuning parameters used in each experimental setting.

\begin{table}[H]
\centering
\caption{Tuning parameters used in Experiments 1 and 2.}
\label{tab:tuning_parameters}

\begin{tabular}{llcccc}
\toprule
Experiment & Method & Parameter & $L$ & $M$ & $\alpha$ \\
\midrule

\multirow{4}{*}{Exp. 1 (Observation-based)} 
& LF-IBIS (ESS) & $\mu$ 
& 50{,}000 & 50 & 0.99 \\

& LF-IBIS (UP) & $\mu$ 
& 50{,}000 & 50 & 0.98 \\

\cmidrule{2-6}

& LF-IBIS (ESS) & $\beta$ 
& 50{,}000 & 50 & 0.99 \\

& LF-IBIS (UP) & $\beta$ 
& 50{,}000 & 50 & 0.98 \\

\midrule

\multirow{4}{*}{Exp. 2 (Utility-based)} 
& LF-IBIS (ESS) & $\mu$ 
& 100{,}000 & 50 & 0.92 \\

& LF-IBIS (UP) & $\mu$ 
& 100{,}000 & 50 & 0.95 \\

\cmidrule{2-6}

& LF-IBIS (ESS) & $\beta$ 
& 100{,}000 & 50 & 0.93 \\

& LF-IBIS (UP) & $\beta$ 
& 100{,}000 & 50 & 0.95 \\

\bottomrule
\end{tabular}
\end{table}

\subsection{Experimental Setting}
Response-Adaptive Randomization (RAR) is a special adaptive design used in clinical trials to update the allocation to treatments of incoming study participants to favour the most promising and beneficial treatments as indicated by interim results \cite{wilson2025rar,robertson2023rar}. The application of RL to RAR is promising for the potentiality of RL to allow dynamic updates to randomization probabilities based on real‐time predictions of treatment success \cite{deliu2025errorcontrol, Merrell2023}. In this setting, the environment represents patients, with states corresponding to their treatment responses. We assume there are $N$ patients entering the trial sequentially, and that each patient’s response is observed before the allocation of the next patient. The agent’s goal is to assign the $n$-th patient ($1\leq n\leq N$) to treatment or placebo in order to maximize the number of disease remissions.

Here, the action space is $\mathcal{A} = \{0,1\}$, and each action $A_n$ is a Bernoulli random variable representing the assignment of the $n$-th patient to treatment ($A_n=1$) or control ($A_n=0$). The policy $\pi \in [0,1]$ corresponds to the Bernoulli parameter governing the i.i.d. sequence of random variables $\{A_n\}_{n=1}^N$.

The environment is characterized by two possible states:  $\mathcal{S} = \{0,1\}$. If the $n$-th patient achieves remission, the state is $S_n = 1$; otherwise, $S_n = 0$. The responses of the environment are two Bernoulli random variables: $ Z_0 \sim \text{Ber}(\mu_0)$ and $ Z_1 \sim \text{Ber}(\mu_1)$, where $Z_0$ denotes  $S_n\mid A_n=0$ and $Z_1$ denotes  $S_n\mid A_n=1$, while $\mu_0$ and $\mu_1$  are the probabilities of remission
 in the control and treatment group, respectively.
 
The reward depends on both the observed state and the treatment decision.
Let us suppose that in our experiments a collateral effect may occur with probability $0.7$ when treatment is administered. In the absence of treatment, it is assumed that the reward coincides with the state, namely $$R_n(a_n=0, s_n=s) =s  \qquad \text{for }s\in\{0,1\}.$$ 
When treatment is administered, conditional on treatment, the reward is assumed to be:

\[
R_n(a_n=1, s_n=0) =
\begin{cases}
-0.2, & \text{if a collateral effect occurs}, \\
0,    & \text{if no collateral effect occurs}.
\end{cases}
\]

\[
R_n(a_n=1, s_n=1) =
\begin{cases}
0.8, & \text{if a collateral effect occurs}, \\
1,   & \text{if no collateral effect occurs}.
\end{cases}
\]
Note that, while in the absence of treatment the reward is deterministic, under treatment the reward is a random variable that depends on both the observed state $s_n$ and the treatment decision $a_n$, as well as on the independently sampled collateral effect.

The observed history at time $n$ (corresponding to the $n$-th patient) is defined as $x_{1:n} = \{(a_1,s_1,r_1), \dots , (a_n,s_n,r_n)\} \in \mathcal{X}^n$, where $r_n$ denotes the realization for  $R_n(a_n,s_n)$. 

Our objective is to perform inference on $\mu_0$ and $\mu_1$ in order to learn the optimal policy, conditional on the history up to time $n$:

\begin{center}
    $p(\pi^*, \mu_0, \mu_1 \mid x_{1:n})$.
\end{center}

We assume uninformative Beta priors for the parameters, i.e., $\mu_0, \mu_1, \pi \sim \text{Beta}(1,1)$. To move the particles, we construct proposal distributions of the form $q(\mu_{0,i-1}^{(\ell)}, \cdot)$, which at iteration $i$ are defined as Beta distributions whose parameters are chosen via the method of moments. This ensures that the proposal is centered at the previously accepted particle $\mu_{0,i-1}^{(\ell)}$ and that its variance matches twice the empirical variance of the set of previously accepted particles $\{\mu_{0,i-1}^{(\ell)}\}_{\ell=1}^L$.
An analogous construction is used for $q(\mu_{1,i-1}^{(\ell)}, \cdot)$. Although such Beta proposal distributions are not standard, their design is inspired by the adaptive proposal strategy widely used for multivariate normal proposals, as popularized in the adaptive ABC framework of Beaumont et al.~\cite{beaumont2009adaptive}. Full details of the proposal definition are provided in the Appendix \ref{sec:beta_proposal}.

The clinical application described so far can also be faced 
by expressing $\mu_0$ and $\mu_1$ as functions of alternative parameters:
\[
    \mu_0 = \frac{1}{1 + e^{-\beta_0}}, 
    \qquad
    \mu_1 = \frac{1}{1 + e^{-(\beta_0 + \beta_1)}},
\]
with $\beta_0, \beta_1 \in \mathbb{R}$. 
This inverse-logit (expit
) transformation allows us to work with parameters defined on the entire real line. Under this parametrization, one can adopt normal priors for $\beta_0$ and $\beta_1$; 
however, these priors are not conjugate, and as a result the posterior distributions do not have closed-form expressions. An additional advantage of this logistic representation is that it naturally accommodates the inclusion of covariates (e.g., patient-specific characteristics) \cite{Deliu}. Under this parametrization, the target distribution becomes the posterior $p (\pi, \beta_0, \beta_1 \mid x_{1:n})$
 from which $p(\pi,\mu_0, \mu_1 \mid x_{1:n})$ can be recovered through the expit transformation.
The prior distributions used in this 
setting are
$
\beta_0 \sim \mathcal{N}(0, 16), 
\beta_1 \sim \mathcal{N}(2, 16).
$
The proposal distribution at iteration \(i\) is defined as a bivariate Normal distribution centered at the empirical mean of the 
particles alive at the previous iteration $\hat{m}_{i-1}$, with covariance matrix $\widehat{\Sigma}_{i-1}$ equal to twice their  empirical covariance ~\cite{beaumont2009adaptive}:
\[
q\big( (\beta_{0,i}^{(\ell)}, \beta_{1,i}^{(\ell)}) \mid \{ (\beta_{0,i-1}^{(\ell)}, \beta_{1,i-1}^{(\ell)}) \}_{\ell=1}^L \big) 
= \mathcal{N}\big( \hat{m}_{i-1}, \; 2 \widehat{\Sigma}_{i-1} \big).
\]



\subsection{Experimental Results \label{sec:experiment}}

In the following, we present results for the described application in both its original formulation and the 
version based on logistic transformation. The observed history $x_{1:N}$ was generated using the true values \( \mu_0^{\text{true}} = 0.3 \) and \( \mu_1^{\text{true}} = 0.8 \), that, under the logistic formulation, correspond to the following values for $\beta_0^{\text{true}}$ and $\beta_1^{\text{true}}$:
$\beta_0^{\text{true}} = \operatorname{logit}(\mu_0^{\text{true}}) \approx -0.85, 
\qquad
\beta_1^{\text{true}} = \operatorname{logit}(\mu_1^{\text{true}}) - \beta_0^{\text{true}} \approx 2.24.
$

In the first experiment with observation-based statistics and Hellinger distance, the final threshold $\epsilon_{\mathrm{final}}$ was determined from the AR-ABC simulations. Specifically, $\epsilon_{\mathrm{final}}$ was chosen as the empirical $0.15\%$ quantile of the simulated distance distribution. 
In more challenging settings, such as the utility-based experiment, AR-ABC becomes substantially less efficient. In such cases, $\epsilon_{\mathrm{final}}$ is calibrated empirically by considering candidate tolerance values and assessing the quality of the resulting posterior approximations. Whenever the exact Bayesian posterior is available, it is used as a benchmark to evaluate the agreement between the approximate and exact posteriors.


\subsubsection{Experiment 1: Hellinger distance on observation-based statistics}
\label{sec:experiment1}
Figure~\ref{fig:trial_mu_1} presents the posterior distributions for $\mu_0$ and $\mu_1$ obtained using LF-IBIS with ESS and LF-IBIS with UP (left panels), along with the corresponding learned policies (right panels). Both variants of the LF-IBIS algorithm yield posterior approximations that are consistent with exact Bayesian inference, similarly AR-ABC approach.
The final threshold value was $\epsilon_{\mathrm{final}}=0.02165$. With this threshold, AR-ABC was run with 3M simulated samples, yielding 4,936 accepted particles.
For the LF-IBIS experiments (both ESS- and UP-based), the initial history length was set to 3, the initial threshold to $\epsilon = 1$, and the final history length to 48. The tuning parameters $M$, $L$, and $\alpha$ were set as reported in Table~\ref{tab:tuning_parameters}. The average execution time was approximately 23 minutes for LF-IBIS with ESS and 29 minutes for LF-IBIS with UP. The total number of iterations required to reach the target tolerance was 94 for the ESS-based version and 89 for the UP-based version.

\begin{figure}[htbp]
    \begin{center}
    \begin{tabular}{c}
    \includegraphics[scale=0.175]{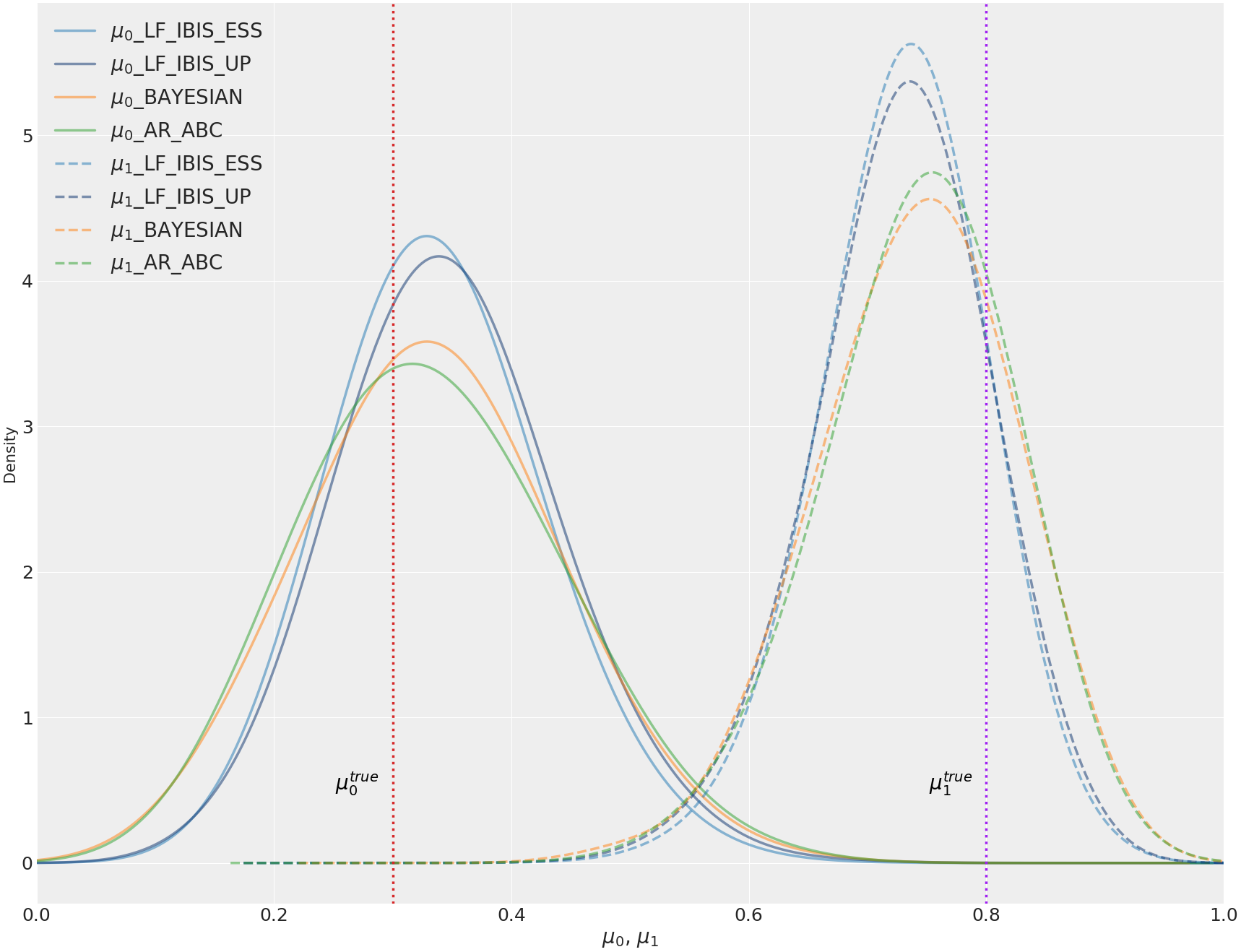}
    \includegraphics[scale=0.175]{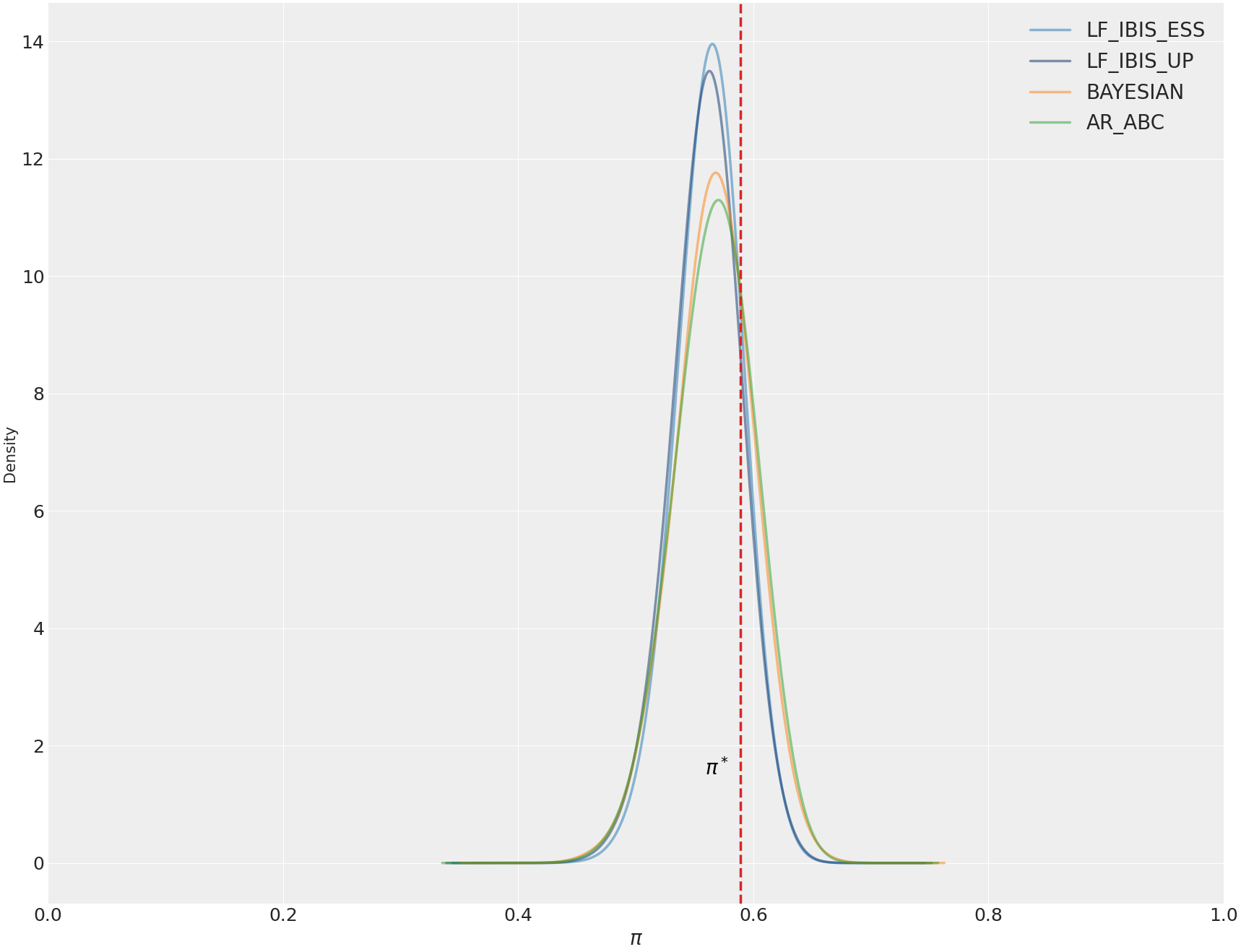} \\
    \end{tabular}
    
    \caption{
    Posterior distributions of $\mu$ (left panel) and the corresponding policies (right panel), when using Hellinger distance on observation-based statistics.
    Light blue: LF-IBIS with ESS.
    Dark blue: LF-IBIS with UP.
    Orange: exact Bayesian inference.
    Green: AR-ABC.
    The dashed vertical lines indicate $\mu_0^{\text{true}}$ and $\mu_1^{\text{true}}$ in the left panel,
    and the corresponding true posterior optimal policy $\pi^*$ in the right panel,
    computed via policy iteration using $\mu_0^{\text{true}}$ and $\mu_1^{\text{true}}$.
    }
    \label{fig:trial_mu_1}
    \end{center}
\end{figure}

It is worth noting that the posterior distributions obtained with LF-IBIS appear to be slightly more concentrated than those produced by both exact Bayesian inference and the AR-ABC method. This difference is expected, as both exact Bayesian inference and AR-ABC are performed in an offline setting using the complete observed history, whereas LF-IBIS processes observations sequentially as they become available. Consequently, the posterior distributions obtained by these approaches cannot be expected to coincide exactly. Nevertheless LF-IBIS yields posterior distributions that are very close to those of the offline methods and remain highly consistent in shape.

Figure~\ref{fig:trial_beta_1} reports the results obtained when the problem is reparameterized using the logistic transformation. In this setting, the reparameterization makes exact Bayesian inference intractable. The parameter settings for both LF-IBIS and AR-ABC are identical to those used in Figure~\ref{fig:trial_mu_1}, and the same considerations regarding the quality of the approximations apply here as well. With the same threshold value $\epsilon = 0.02165$, the AR-ABC procedure accepted 1,355 samples out of 3M simulated draws. The average execution time was approximately 35 minutes with 125 iterations for LF-IBIS with ESS, and approximately 36 minutes with 108 iterations for LF-IBIS with UP. Another noteworthy point is that the UP strategy seems to require fewer iterations to reach the target tolerance~$\epsilon$, in both the current example reparameterized via the logistic transformation and in the previous one.

\begin{figure}[H]
    \begin{center}
    \begin{tabular}{c}
    \includegraphics[scale=0.175]{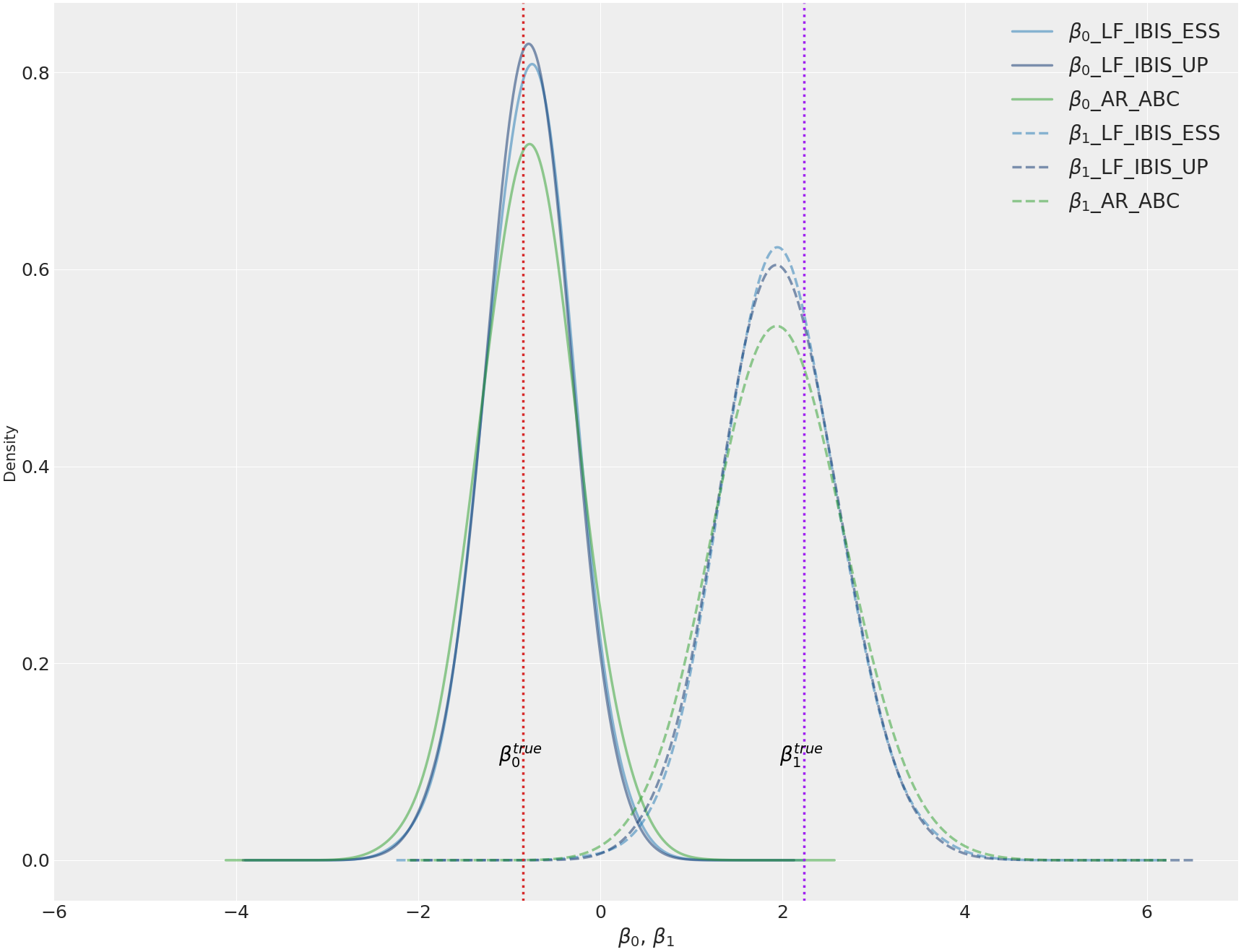}
    \includegraphics[scale=0.175]{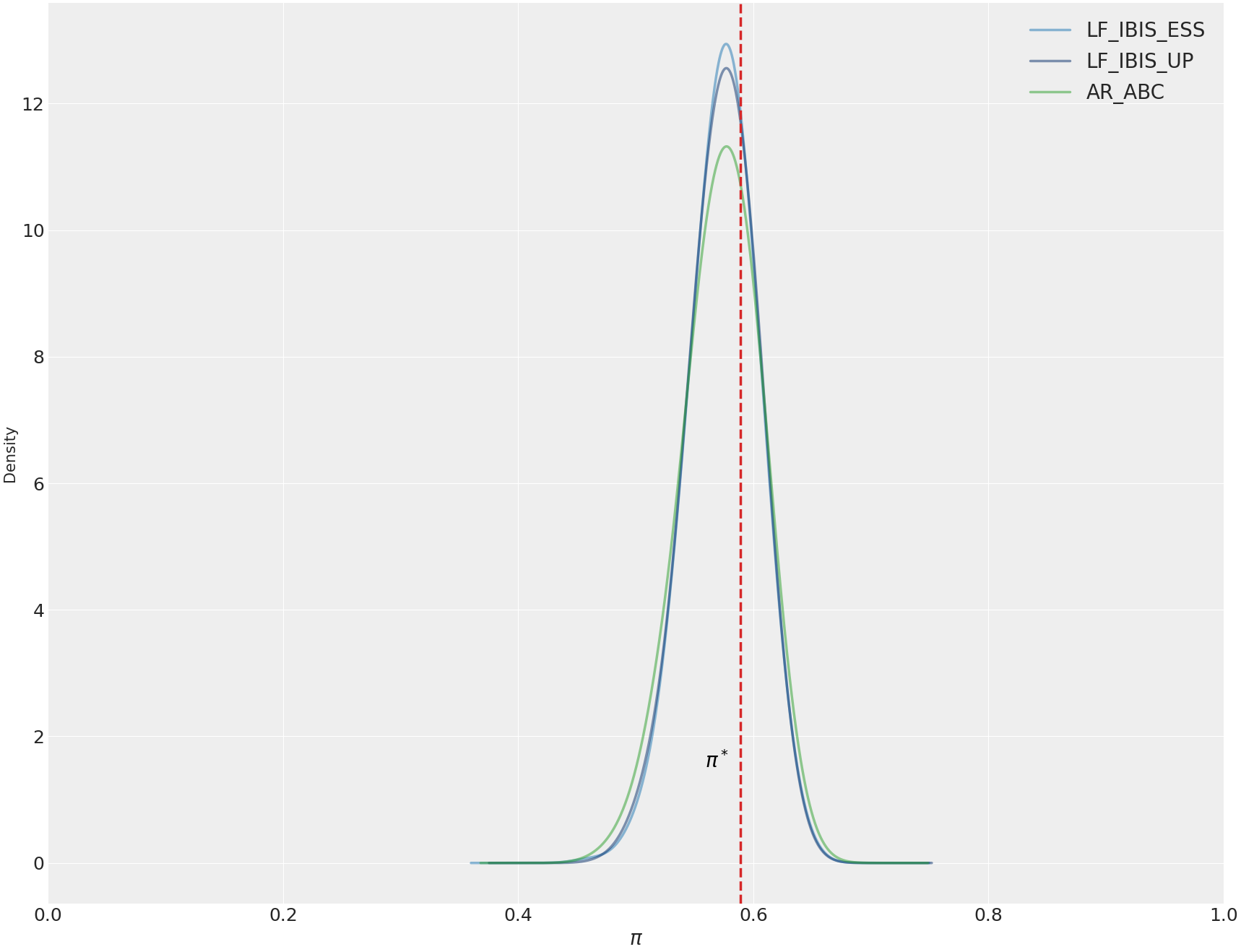}\\
    \end{tabular}
    \caption[]{\footnotesize 
    Posterior distributions of $\beta_0$ and $\beta_1$ (left panel) and the corresponding policies (right panel), when using Hellinger distance on observation-based statistics.
    Light blue: LF-IBIS with ESS.
    Dark blue: LF-IBIS with UP.
    Green: AR-ABC.
    The dashed red vertical line indicates $\beta_0^{\text{true}}$ and $\beta_1^{\text{true}}$ in the left panel,
    and the corresponding true posterior optimal policy $\pi^*$ in the right panel,
    computed via policy iteration using $\beta_0^{\text{true}}$ and $\beta_1^{\text{true}}$.}
    \label{fig:trial_beta_1}
    \end{center}
\end{figure}

\subsubsection{Experiment 2: Euclidean distance on utility-based statistics}
\label{sec:experiment2}
Figure~\ref{fig:trial_mu_2} shows the results obtained using utility-based statistics. In this case, all results were produced with a final threshold value of $\epsilon = 2.1 \times 10^{-5}$.
For LF-IBIS with ESS, the initial history length was set to 3, the initial threshold to $\epsilon = 0.5$, and the final history length to 48. The tuning parameters $M$, $L$, and $\alpha$ were set as reported in Table~\ref{tab:tuning_parameters}. The average execution time was approximately 16 minutes for LF-IBIS with ESS (with 103 iterations) and approximately 27 minutes for LF-IBIS with UP (with 110 iterations).

\begin{figure}[H]
    \begin{center}
    \begin{tabular}{c}
    \includegraphics[scale=0.175]{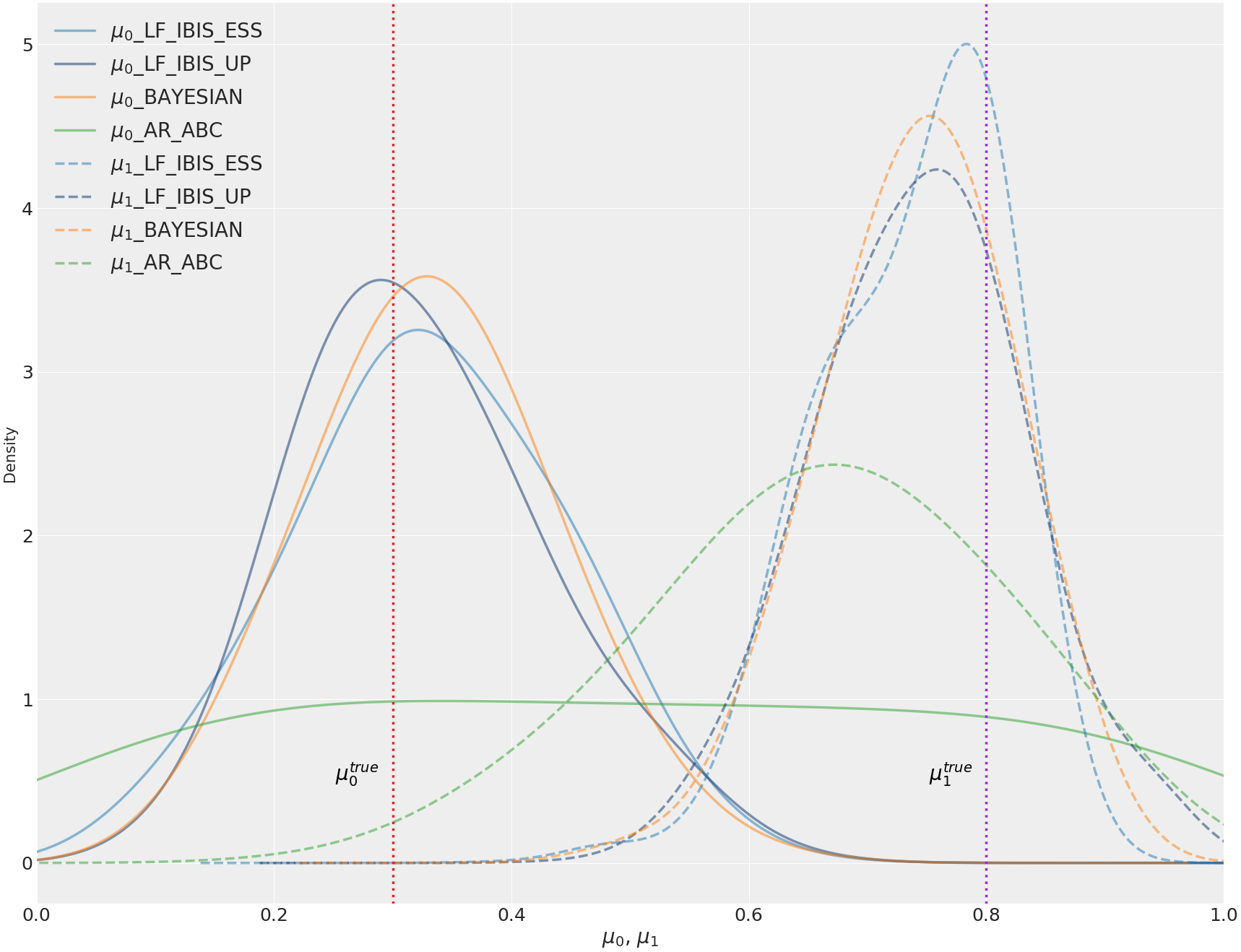}
    \includegraphics[scale=0.175]{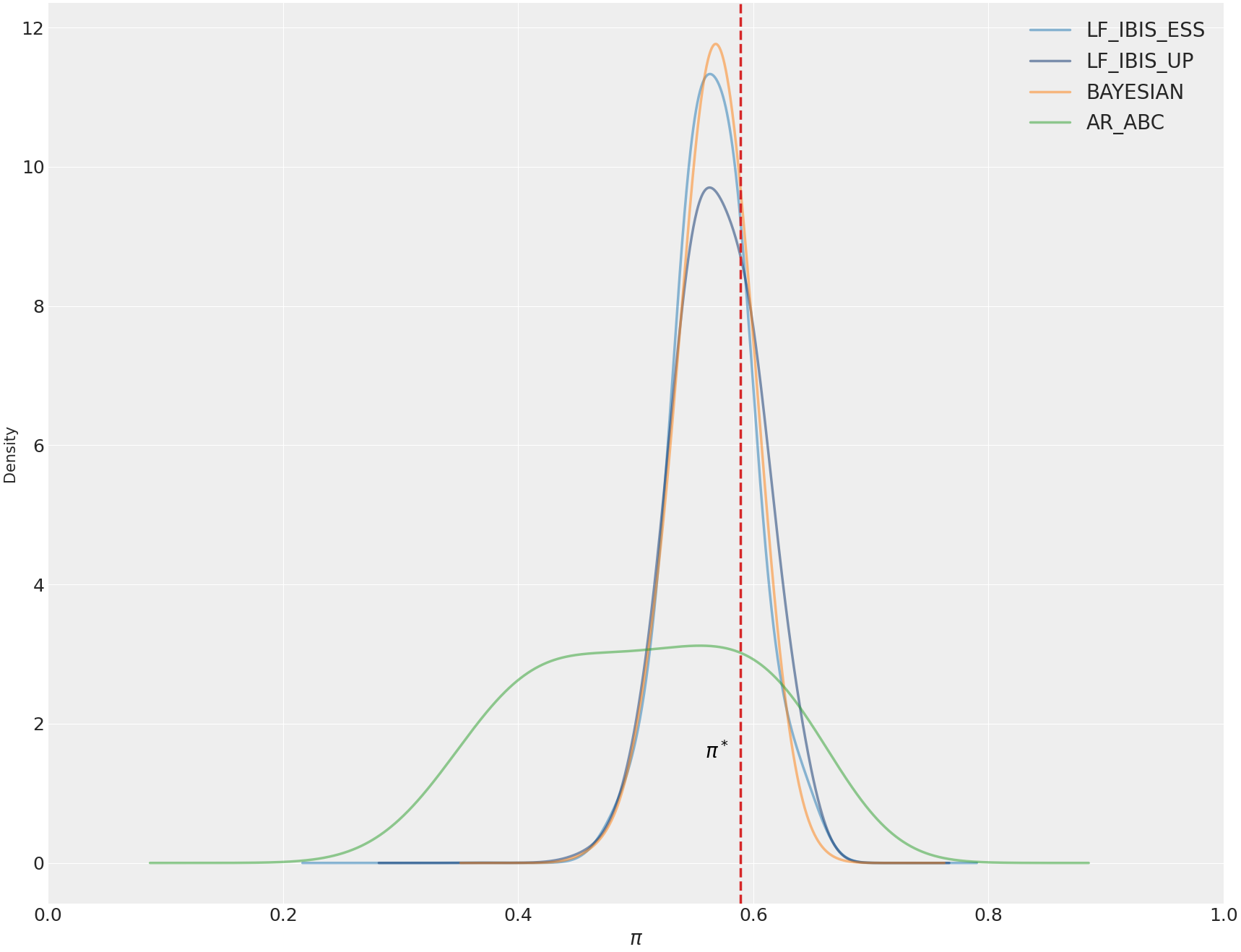} \\
    \end{tabular}
    \caption[]{\footnotesize Posterior distributions of $\mu$ (left panel) and the corresponding policies (right panel), when using Euclidean distance on utility-based statistics.
    Light blue: LF-IBIS with ESS.
    Dark blue: LF-IBIS with UP.
    Orange: exact Bayesian inference.
    Green: AR-ABC.
    The dashed red vertical line indicates $\mu_0^{\text{true}}$ and $\mu_1^{\text{true}}$ in the left panel,
    and the corresponding true posterior optimal policy $\pi^*$ in the right panel,
    computed via policy iteration using $\mu_0^{\text{true}}$ and $\mu_1^{\text{true}}$.
    }
    \label{fig:trial_mu_2}
    \end{center}
\end{figure}

In this experiment, unlike in the previous one, the posterior approximation obtained with LF-IBIS appears to be more accurate than that produced by AR-ABC. In the previous experiment, observation-based statistics were employed; these are constructed from empirical distributions and therefore tend to preserve more information about the underlying parameters~$\mu$. In contrast, the utility-based statistics used here constitute a suboptimal summary of the data, as they aggregate information into coarse measures and consequently convey less information about~$\mu$. 
In this setting, the sequential (online) nature of LF-IBIS, as opposed to the fully offline approach adopted by AR-ABC—which processes the entire observed history at once—appears to effectively compensate for the reduced informativeness of the utility-based statistics, resulting in a more coherent posterior approximation.

Despite being less informative than observation-based summaries, utility-based statistics can be considerably more practical in applied settings. Our experimental framework is intentionally simplified, as we consider an MDP for which informative statistics of the process are known. In more realistic applications, however, such knowledge may be unavailable. For instance, in 
model-free contexts or in problems involving complex or partially observed dynamics, identifying sufficient statistics can be extremely challenging. In such cases, utility-based statistics offer a flexible alternative.

An additional consequence of the reduced informativeness of the utility-based statistics is the need for a larger particle population to achieve stable inference. Specifically, we employ $L = 100{,}000$ particles, compared to $L = 50{,}000$ in the observation-based setting. Nevertheless, utility-based statistics are computationally less demanding, as they do not require the construction of empirical distributions at each iteration. This results in shorter execution times: LF-IBIS with utility-based statistics completes in approximately 16--27 minutes (for both the ESS- and UP-based variants), compared to roughly 23--29 minutes for the observation-based setting.

Finally, the tuning parameter $\alpha$ was reduced relative to the previous experiments. In the current study, we set $\alpha = 0.92$ for the ESS-based variant and $\alpha = 0.95$ for the UP-based variant. Smaller values of $\alpha$ allow particles with negligible importance weights to be pruned and replaced earlier in the procedure, thereby accelerating adaptation of the particle population to the observed data.




Figure~\ref{fig:trial_beta_2} reports the results obtained with utility-based statistics when the problem is reparameterized through the logistic transformation. The parameter values for LF-IBIS and AR-ABC are the same as in Figure~\ref{fig:trial_mu_2}, and the same considerations regarding the results apply here as well. For the AR-ABC, 693 samples are accepted. The execution time is $\approx 18$ minutes with 109 iterations for the ESS version and $\approx 29$ minutes with 116 iterations for the UP version.

\begin{figure}[H]
    \begin{center}
    \begin{tabular}{c}
    \includegraphics[scale=0.175]{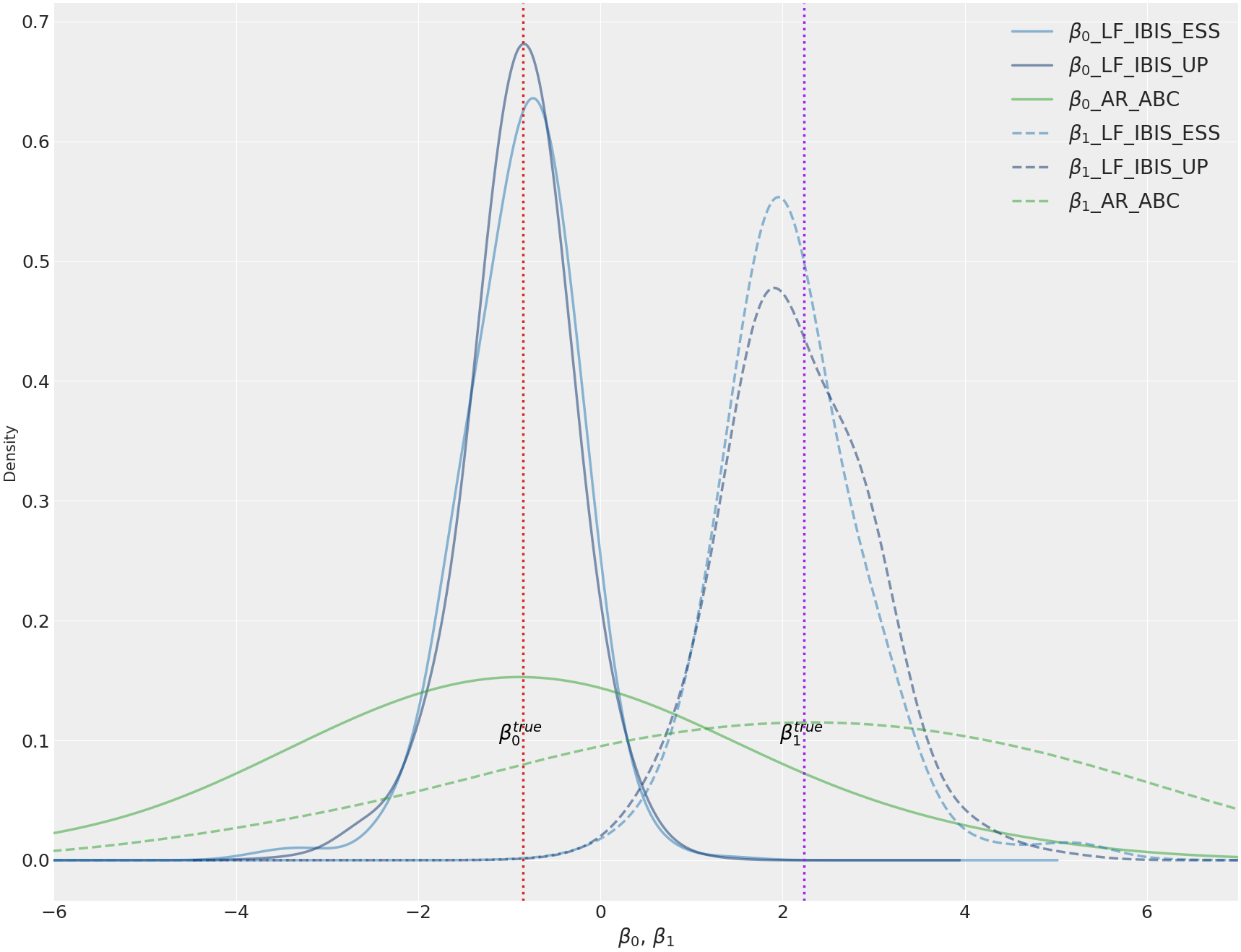}
    \includegraphics[scale=0.175]{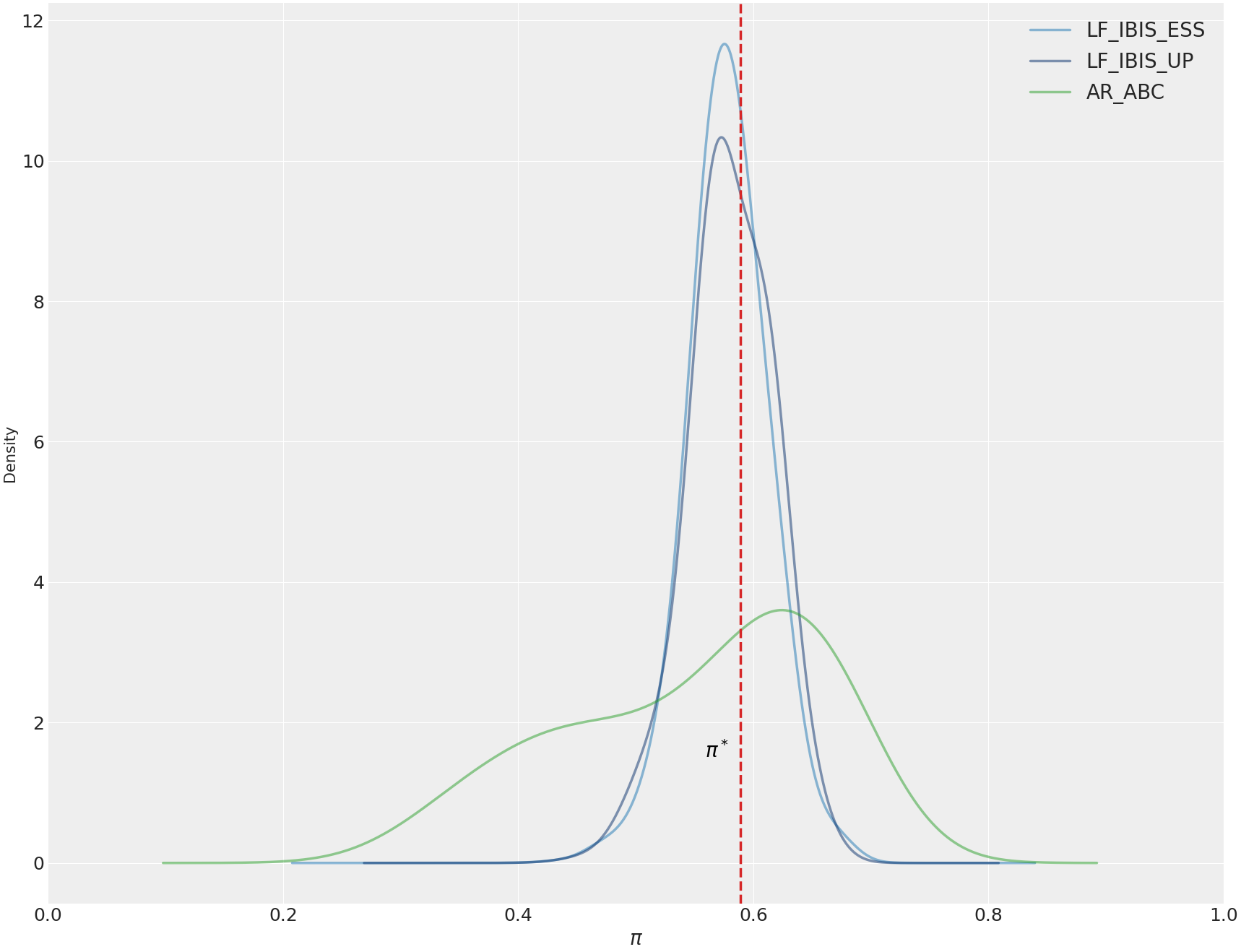}\\
    \end{tabular}
    \caption[]{\footnotesize 
    Posterior distributions of $\beta_0$ and $\beta_1$ (left panel) and the corresponding policies (right panel), when using Euclidean distance on utility-based statistics.
    Light blue: LF-IBIS with ESS.
    Dark blue: LF-IBIS with UP.
    Green: AR-ABC.
    The dashed red vertical line indicates $\beta_0^{\text{true}}$ and $\beta_1^{\text{true}}$ in the left panel,
    and the corresponding true posterior optimal policy $\pi^*$ in the right panel,
    computed via policy iteration using $\beta_0^{\text{true}}$ and $\beta_1^{\text{true}}$.}
    \label{fig:trial_beta_2}
    \end{center}
\end{figure}

Table~\ref{tab:experiment_settings} summarizes the experimental settings adopted in the two experiments. The table includes both algorithmic hyperparameters (e.g., particle population size and tuning parameters) and execution-related measures, such as the number of LF-IBIS iterations and runtime, as well as the final ESS for LF-IBIS and the total number of simulated samples and accepted particles for AR-ABC.
It is important to note that in both the observation-based and utility-based experiments, although the UP variant often converges in the same or fewer iterations compared to the ESS-based approach, it is generally more time-consuming overall. This is because each application of the bisection method requires an additional resampling step.

\begin{table}[H]
\centering
\caption{Summary of experimental settings for Experiments 1 and 2.}
\label{tab:experiment_settings}
\resizebox{\textwidth}{!}{
\begin{tabular}{llccccccc}
\toprule
Experiment & Method & Param. & $L$ & $M$ & $\alpha$ & $\epsilon_{\text{final}}$ & Iterations / Time & \makecell{Final ESS \\ AR-ABC Simulated / Accepted} \\

\midrule
\multirow{6}{*}{Exp. 1 (Obs.-based)} 
& LF-IBIS (ESS) & $\mu$ 
& 50{,}000 & 50 & 0.99 & 0.02165 & 94 / $\sim$23 min & 19576\\

& LF-IBIS (UP) & $\mu$ 
& 50{,}000 & 50 & 0.98 & 0.02165 & 89 / $\sim$29 min & 18701\\

& AR-ABC & $\mu$ 
& -- & -- & -- & 0.02165 & $\sim$13 sec & 3M / 4{,}936 \\

\cmidrule{2-9}
& LF-IBIS (ESS) & $\beta$ 
& 50{,}000 & 50 & 0.99 & 0.02165 & 125 / $\sim$35 min & 5732\\

& LF-IBIS (UP) & $\beta$ 
& 50{,}000 & 50 & 0.98 & 0.02165 & 108 / $\sim$36 min & 5758\\

& AR-ABC & $\beta$ 
& -- & -- & -- & 0.02165 & $\sim$18 sec & 3M / 1{,}355 \\

\midrule
\multirow{6}{*}{Exp. 2 (Utility-based)} 
& LF-IBIS (ESS) & $\mu$ 
& 100{,}000 & 50 & 0.92 & $2.1\times10^{-5}$ & 103 / $\sim$16 min & 260 \\

& LF-IBIS (UP) & $\mu$ 
& 100{,}000 & 50 & 0.95 & $2.1\times10^{-5}$ & 110 / $\sim$27 min & 449\\

& AR-ABC & $\mu$ 
& -- & -- & -- & $2.1\times10^{-5}$ & $\sim$4 min & 150M / 725 \\

\cmidrule{2-9}
& LF-IBIS (ESS) & $\beta$ 
& 100{,}000 & 50 & 0.93 & $2.1\times10^{-5}$ & 109 / $\sim$18 min & 381 \\

& LF-IBIS (UP) & $\beta$ 
& 100{,}000 & 50 & 0.95 & $2.1\times10^{-5}$ & 116 / $\sim$29 min & 497\\

& AR-ABC & $\beta$ 
& -- & -- & -- & $2.1\times10^{-5}$ & $\sim$6 min & 150M / 693 \\

\bottomrule
\end{tabular}
}
\end{table}

As a final assessment, to quantify the discrepancy between the approximate posterior distributions on $\mu$ produced by AR-ABC and by the two variants of LF-IBIS (ESS- and UP-based) and the exact Bayesian posterior, we use the \emph{Energy Distance}~\cite{SzekelyRizzo2004} for both experiments. The energy distance is a metric that measures the distance between probability distributions based on pairwise distances between samples.
For each method, we report the mean and standard deviation of the Energy distance computed over ten independent runs of the algorithm, all performed using the same observed history as in the experiments above. 
The Energy distance is reported only for the non-reparameterized versions of the experiments, since under the logistic reparameterization the exact Bayesian posterior is not available. Overall, the UP-based variant tends to produce posterior distributions that are closer to the exact Bayesian posterior. Moreover, the approximations obtained with UP exhibit lower variability across independent runs, and this difference is particularly pronounced in Experiment~2.

In Appendix~\ref{sec:supplementary2}, we also report the KL divergences between the posteriors of $\mu$ and $\pi$.


\begin{table}[H]
\centering
\footnotesize
\caption{Mean (standard deviation) of the Energy distance computed over ten independent runs of the algorithms with respect to the exact Bayesian posterior. Results are reported for inference on the parameter $\mu$ and the policy $\pi$.}
\medskip
\label{tab:energy_comparison}
\begin{tabular}{lcccc}
\toprule
\textbf{Method} 
& \multicolumn{2}{c}{\textbf{Exp. 1 (Obs.-based)}} 
& \multicolumn{2}{c}{\textbf{Exp. 2 (Utility-based)}} \\
\cmidrule(lr){2-3} \cmidrule(lr){4-5}
& $\mu$ & $\pi$ & $\mu$ & $\pi$ \\
\midrule
LF-IBIS (ESS) 
& 0.0033 (0.0005) & 0.0342 (0.0040)  & 0.0178 (0.0172) & 0.0676 (0.0402) \\
LF-IBIS (UP) 
& 0.0026 (0.0003) & 0.0265 (0.0029) & 0.0031 (0.0025) & 0.0273 (0.0131) \\
AR-ABC 
& 0.0001 (0) & 0.0040 (0.0011) & 0.1335 (0.0076) & 0.2199 (0.0064) \\
\bottomrule
\end{tabular}
\end{table}

\subsection{Updating policy from posterior distribution}
\label{sec:updating_policy}

The posterior distribution over policies provides valuable information for deciding when it is advantageous to update the current policy. In this section, policy updates are performed according to the criterion introduced in Section~\ref{subsec:posterior_policy_update}.

At each iteration, samples from the posterior distribution of optimal policies are used to evaluate the Bayesian Mean Squared Error (BMSE) for both the current policy and a candidate policy drawn from the posterior. The policy with the lower BMSE is selected, as it is, on average, closer to the set of plausible optimal policies.

The Bayesian Mean Squared Error for a policy $\pi$ is defined in Eq \eqref{eq:BMSE}

Intuitively, when the posterior distribution assigns high probability mass to the current policy, exploration of the environment is encouraged; otherwise, the policy is updated to exploit the information contained in the posterior.

In Figure~\ref{fig:v_fun_for_various_policies}, we illustrate the value function as a function of the policy parameter $\pi$ for Experiment~1, described in Section~\ref{sec:experiment1}, when LF-IBIS is run using the ESS-based criterion for decreasing the tolerance threshold $\epsilon$. The blue curve represents the posterior mean of the value function, while the shaded region corresponds to the $90\%$ posterior credibility bands. The dashed red vertical line indicates the true posterior optimal policy, computed via policy iteration using the true environment parameters $\mu_0^{\text{true}}$ and $\mu_1^{\text{true}}$. The posterior distribution of the value function is concentrated around a mean function that assigns higher values to policies in the neighborhood of the true optimal policy. This highlights a novel aspect of our approach, namely the ability to obtain a posterior distribution over the value function by propagating uncertainty in the parameters governing the environment dynamics.

\begin{figure}[H]
    \begin{center}
    \begin{tabular}{c}
    \includegraphics[scale=0.35]{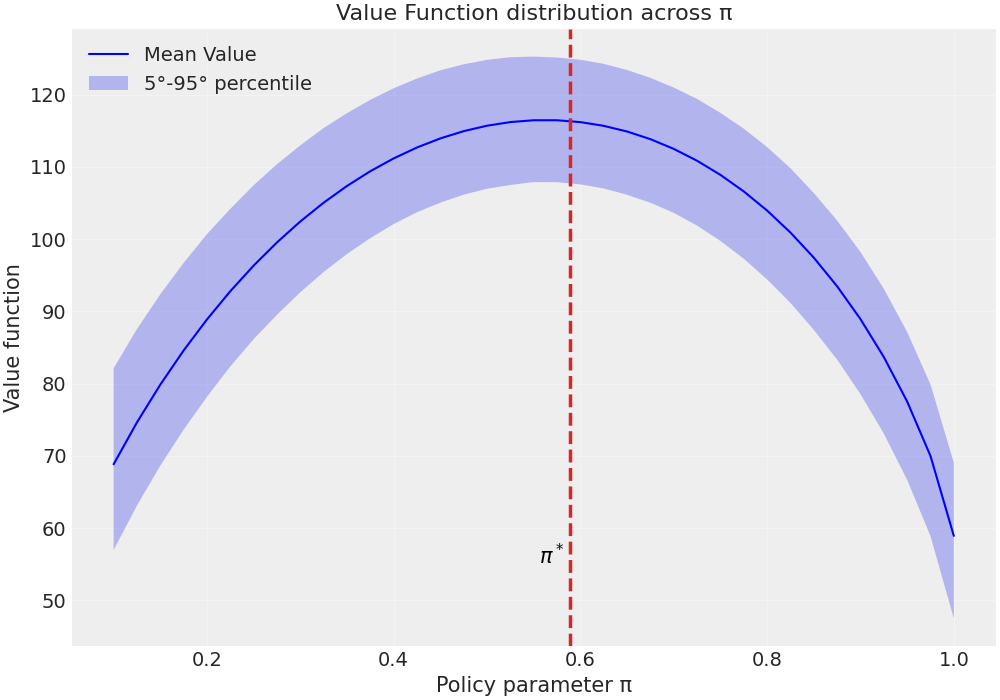}
    \end{tabular}
    \caption[]{\footnotesize Value function when the policy $\pi$ varies from 0 to 1. Blue curve: posterior mean of the value function with $90\%$ posterior credibility bands.
    }
    \label{fig:v_fun_for_various_policies}
    \end{center}
\end{figure}

In Figure~\ref{fig:policies_switch_trial_2}, we compare the value function obtained when the policy is kept fixed throughout the experiment with that obtained after switching to the updated policy. The results refer to Experiment~1 (Section~\ref{sec:experiment1}) ESS version. We also report the corresponding number of patients assigned to the treatment group, with and without policy switching. Credible bands are computed using posterior samples of the environment parameters~$\mu$.

The figure shows that, from a practical perspective, switching the policy is relevant. Whenever an improvement in the value function is observed, a corresponding increase in the number of patients assigned to the treatment group is also observed. The variability around the value function reflects the posterior uncertainty of the environment parameters~$\mu$: when the history length increases, this variability progressively decreases.

In this setting, it is possible to plot the value function because the state space consists of only two states, $\{0,1\}$, and—by construction—the value function does not differ between them. When the blue curve dominates the red one, a clear separation between the two value functions is observed, whereas the converse is not true. Moreover, the blue curve almost never falls below the red one, except for a negligible deviation at a single point.

\begin{figure}[H]
    \begin{center}
    \begin{tabular}{c}
    \includegraphics[scale=0.305]{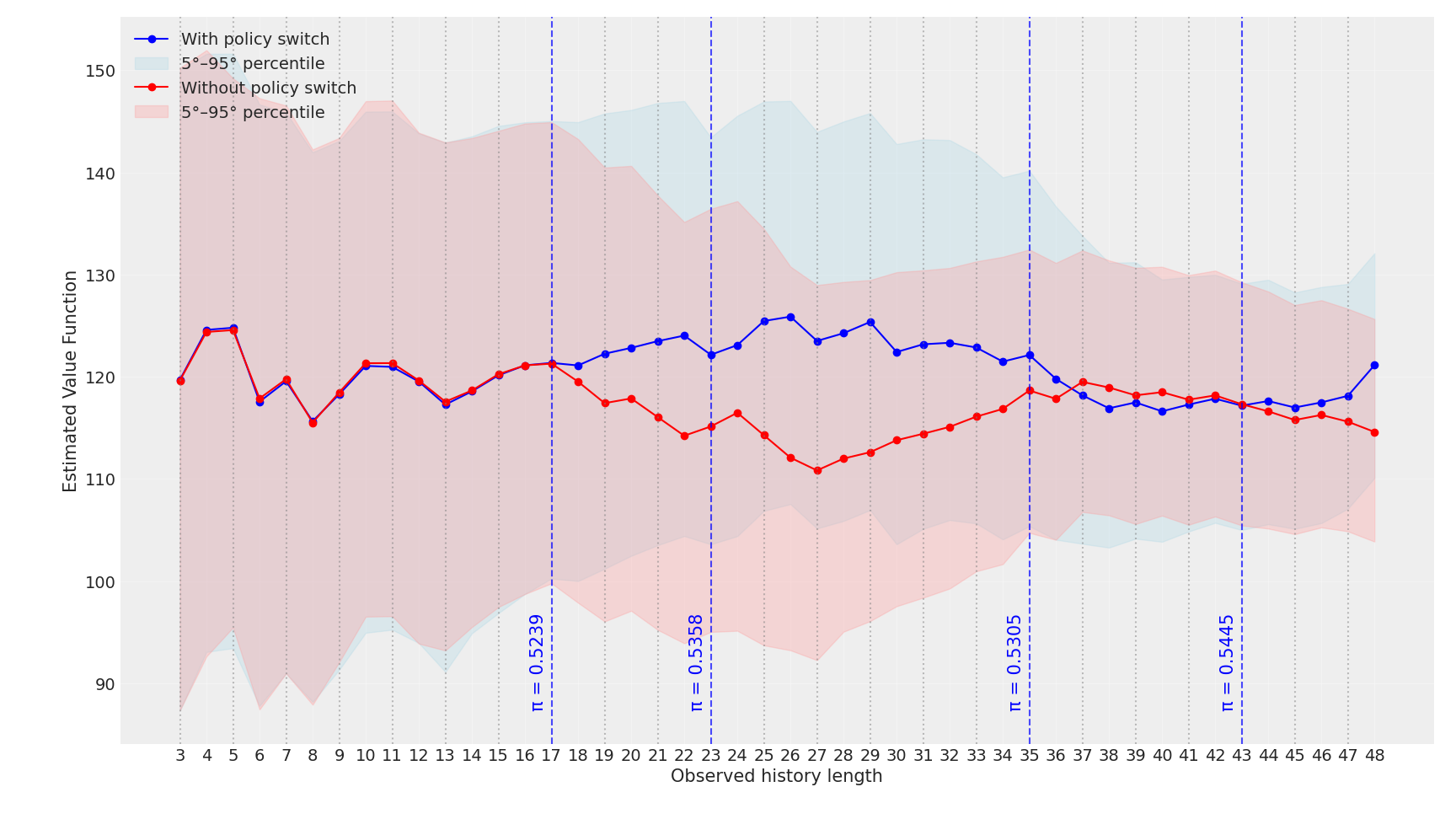}\\
    \includegraphics[scale=0.3]{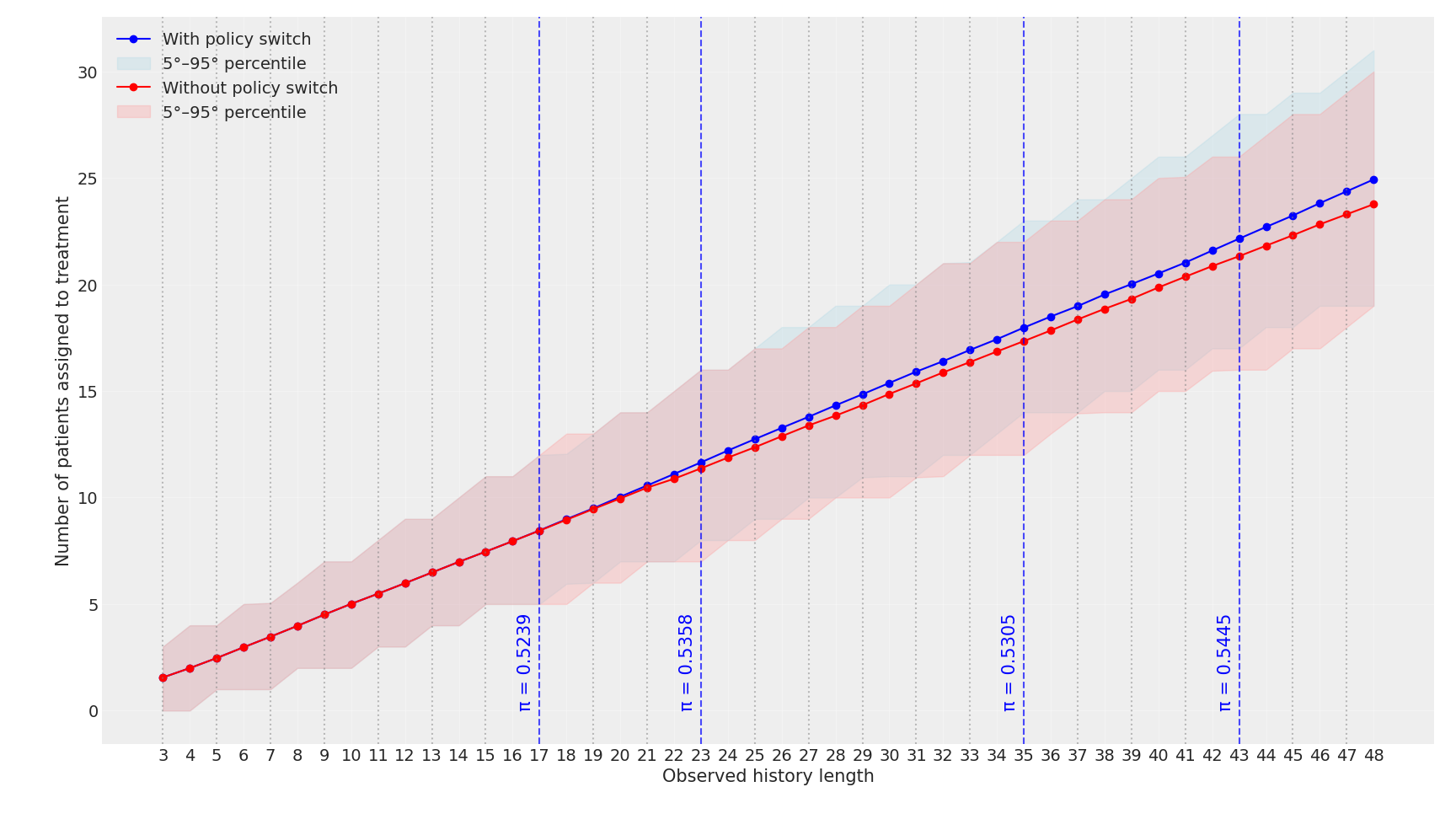}
    \end{tabular}
    \caption[]{\footnotesize  Value function (top panel) and number of patients assigned to treatment (bottom panel) with and without policy switch (see Figure~\ref{fig:trial_mu_1}, ESS version). Here posterior sampling is performed after every 2 new interactions with the environment. Vertical lines indicate decision points: grey lines represent the time steps (i.e., lengths of the observed history) at which a policy switch is evaluated, while blue lines denote the time steps at which a policy switch is actually performed.
    }
    \label{fig:policies_switch_trial_2}
    \end{center}
\end{figure}

\section{Discussion and Conclusion} \label{sec:conclusion}

In this work, we introduced a fully Bayesian Reinforcement Learning (fBRL) procedure based on \emph{Likelihood-Free Iterated Batch Importance Sampling} (LF-IBIS), designed to address data scarcity and to support online learning in RL. The proposed framework enables posterior inference over both environment parameters and policies without requiring access to an explicit likelihood.

We evaluated LF-IBIS in a controlled experimental setting where exact Bayesian inference is available, thereby allowing a direct assessment of the quality of the approximation. Despite the typical sources of approximation introduced by ABC—most notably the choice of the tolerance $\epsilon$, distance functions, and summary statistics \cite{Sisson}—our results show that the approximate posteriors produced by LF-IBIS closely match the exact posterior distributions. Moreover, we demonstrated that processing data incrementally yields posterior distributions that are comparable to those obtained when the entire observed history is processed at once, thus confirming the validity of the online formulation.

Beyond the empirical validation, we also established asymptotic properties of the proposed likelihood-free reweighting scheme. Specifically, we showed that the ABC approximation provides a consistent approximation of the predictive likelihood and that the resulting importance weights converge almost surely to those of the exact IBIS algorithm as the ABC tolerance tends to zero and the number of simulated histories increases. 

A further contribution of this work is the comparison between observation-based and utility-based summary statistics. Observation-based statistics, constructed from empirical outcome distributions, are more informative and lead to highly accurate posterior approximations. However, they rely on knowledge of sufficient statistics, which may not be available in realistic or complex environments. Utility-based statistics, while less informative due to their aggregated nature, provide a practical and flexible alternative that is computationally more efficient and does not require explicit modeling of the underlying process. Our experimental results indicate that the incremental, sequential nature of LF-IBIS partially compensates for the reduced informativeness of utility-based statistics, yielding shorter execution times while maintaining competitive posterior accuracy.

In addition, we investigated the ability of LF-IBIS to support adaptive decision-making through posterior-driven policy updates. When policies are sampled from their posterior distributions, the algorithm implements a Bayesian exploration--exploitation strategy: early posterior uncertainty promotes exploratory behavior, while posterior concentration gradually shifts the agent toward exploitation. This demonstrates that LF-IBIS not only performs inference over environment parameters, but also provides a principled mechanism for policy adaptation in online settings.

More generally, the proposed framework is not limited to MDP with analytically tractable dynamics. Whenever a simulator of the environment is available—even if the underlying process does not admit an explicit MDP formulation—the LF-IBIS methodology can be naturally extended to this model-free RL setting. 

Although in our experimental setting the likelihood can be written analytically—due to the assumption of a known probability of side effects in the reward function—the proposed approach does not fundamentally rely on this property. In more general scenarios, LF-IBIS can be applied in exactly the same way without requiring explicit specification of the transition dynamics or the side-effect generation mechanism, provided that a simulator is available. This highlights the fully likelihood-free nature of the method and its suitability for settings in which analytical modeling is infeasible or undesirable.

Overall, our results indicate that LF-IBIS constitutes a flexible and robust framework for likelihood-free Bayesian Reinforcement Learning, capable of supporting online inference, policy adaptation, and principled uncertainty quantification in complex decision-making problems.

\newpage

\bibliographystyle{plain}
\bibliography{references}
\newpage
\appendix

\section{Supplementary Material for Section~\ref{sec:lf_ibis}} \label{sec:supplementary}
\subsection{Technical Details for ABC--SMC with $M$ Simulations, Kernel Construction and Weight Simplification}
\label{appendix:abc_details}
In \cite{DelMoral}, the authors show the advantage of considering $M>1$ pseudo-datasets for each parameter proposal. In this case, the target distribution becomes
\begin{equation}
\label{eq:abc_target_M}
\tilde{p}_{\epsilon_i}(\mu, \boldsymbol{y}\mid x)
\propto 
p(\mu)
\left( \prod_{m=1}^M f(y_m \mid \mu) \right)
\left(
\frac{1}{M}\sum_{m=1}^M 
\kappa \big(d(y_m,x); \epsilon_i\big)
\right),
\end{equation}
where $\boldsymbol{y} = (y_1, \dots, y_M)$ denotes a vector of $M$ independent realizations of the random variable $Y$.

The corresponding approximated marginal posterior is
\begin{align*}
\tilde{p}_{\epsilon_i}(\mu\mid x)
&\propto 
\int_{\mathcal{X}^M}
p(\mu)
\left(\prod_{m=1}^M f(y_m\mid\mu)\right)
\left(
\frac{1}{M}\sum_{m=1}^M 
\kappa \big(d(y_m,x); \epsilon_i\big)
\right)
\, d\boldsymbol{y} \\
&=
p(\mu)
\int_{\mathcal{X}^M}
f(\boldsymbol{y}\mid\mu)
\left(
\frac{1}{M}\sum_{m=1}^M 
\kappa \big(d(y_m,x); \epsilon_i\big)
\right)
\, d\boldsymbol{y} \\
&=
p(\mu)\,
\mathbb{E}\!\left[
\kappa \big(d(Y,x); \epsilon_i\big)
\right] \\
&=
p(\mu)\,
\Pr\!\big(d(Y,x)\leq\epsilon_i\big).
\end{align*}

Hence, the likelihood function is still replaced by an estimate of 
$\Pr\!\big(d(Y,x)\leq\epsilon_i\big)$. 
This estimator is less variable than the crude single-simulation version, although it is computationally more expensive, since each evaluation requires $M$ pseudo-datasets.

\medskip

\noindent
\cite{DelMoral} also propose a specific choice for the forward and backward kernels.
\begin{itemize}
    
\item The forward kernel 
\(
K_i\big((\mu_{i-1},\boldsymbol{y}_{i-1}),
(\mu_i,\boldsymbol{y}_i)\big)
\)
is a Markov kernel with invariant distribution 
\(
\tilde{p}_{\epsilon_i}(\mu,\boldsymbol{y}\mid x)
\)
and proposal distribution
\[
q\big((\mu_{i-1},\boldsymbol{y}_{i-1}),
(\mu_i,\boldsymbol{y}_i)\big)
=
q(\mu_{i-1},\mu_i)\,
f(\boldsymbol{y}_i \mid \mu_i).
\]

It is defined as
\begin{align*}
&K_i\big((\mu_{i-1},\boldsymbol{y}_{i-1}),
(\mu_i,\boldsymbol{y}_i)\big) \\
&=
\begin{cases}
q(\mu_{i-1},\mu_i)
\displaystyle\prod_{m=1}^M f(y_{i,m}\mid\mu_i)\,
r\big((\mu_{i-1},\mu_i),
(\boldsymbol{y}_{i-1},\boldsymbol{y}_i)\big),
\\[1.2em]
\hfill
\text{if }
\tilde{p}_{\epsilon_i}(\mu_{i-1},\boldsymbol{y}_{i-1}\mid x)
>
\tilde{p}_{\epsilon_i}(\mu_i,\boldsymbol{y}_i\mid x),
\\[0.8em]
q(\mu_{i-1},\mu_i)
\displaystyle\prod_{m=1}^M f(y_{i,m}\mid\mu_i),
\quad \text{otherwise.}
\end{cases}
\end{align*}

where the acceptance rate is
\begin{align*}
&r\big((\mu_{i-1},\mu_i),
(\boldsymbol{y}_{i-1},\boldsymbol{y}_i)\big)
\\
&=
\frac{
p(\mu_i)\,
f(\boldsymbol{y}_i\mid\mu_i)\,
\Big(
\frac{1}{M}\sum_{m=1}^M
\kappa(d(y_{i,m},x);\epsilon_i)
\Big)
q(\mu_i,\mu_{i-1})\,
f(\boldsymbol{y}_{i-1}\mid\mu_{i-1})
}{
p(\mu_{i-1})\,
f(\boldsymbol{y}_{i-1}\mid\mu_{i-1})\,
\Big(
\frac{1}{M}\sum_{m=1}^M
\kappa(d(y_{i-1,m},x);\epsilon_i)
\Big)
q(\mu_{i-1},\mu_i)\,
f(\boldsymbol{y}_i\mid\mu_i)
}
\\[0.8em]
&=
\frac{
p(\mu_i)\,
\Big(
\frac{1}{M}\sum_{m=1}^M
\kappa(d(y_{i,m},x);\epsilon_i)
\Big)
q(\mu_i,\mu_{i-1})
}{
p(\mu_{i-1})\,
\Big(
\frac{1}{M}\sum_{m=1}^M
\kappa(d(y_{i-1,m},x);\epsilon_i)
\Big)
q(\mu_{i-1},\mu_i)
}.
\end{align*}

\item The backward kernel 
\(
L_{i-1}\big((\mu_i,\boldsymbol{y}_i),
(\mu_{i-1},\boldsymbol{y}_{i-1})\big)
\)
is defined as
\begin{align*}
&L_{i-1}\big((\mu_i,\boldsymbol{y}_i),
(\mu_{i-1},\boldsymbol{y}_{i-1})\big)
\\
&=
\frac{
\tilde{p}_{\epsilon_i}(\mu_{i-1},\boldsymbol{y}_{i-1}\mid x)
}{
\tilde{p}_{\epsilon_i}(\mu_i,\boldsymbol{y}_i\mid x)
}
\,
K_i\big((\mu_{i-1},\boldsymbol{y}_{i-1}),
(\mu_i,\boldsymbol{y}_i)\big)
\\[1em]
&=
\frac{
p(\mu_{i-1})
\displaystyle\prod_{m=1}^M f(y_{i-1,m}\mid\mu_{i-1})
\Big(
\frac{1}{M}\sum_{m=1}^M
\kappa(d(y_{i-1,m},x);\epsilon_i)
\Big)
}{
p(\mu_i)
\displaystyle\prod_{m=1}^M f(y_{i,m}\mid\mu_i)
\Big(
\frac{1}{M}\sum_{m=1}^M
\kappa(d(y_{i,m},x);\epsilon_i)
\Big)
}
\\
&\quad \cdot
\frac{
q_i(\mu_{i-1},\mu_i)
\displaystyle\prod_{m=1}^M f(y_{i,m}\mid\mu_i)
p(\mu_i)
\Big(
\frac{1}{M}\sum_{m=1}^M
\kappa(d(y_{i,m},x);\epsilon_i)
\Big)
q_i(\mu_i,\mu_{i-1})
}{
p(\mu_{i-1})
\Big(
\frac{1}{M}\sum_{m=1}^M
\kappa(d(y_{i-1,m},x);\epsilon_i)
\Big)
q_i(\mu_{i-1},\mu_i)
}
\\[1em]
&=
\left(
\prod_{m=1}^M
f(y_{i-1,m}\mid\mu_{i-1})
\right)
q_i(\mu_i,\mu_{i-1}).
\end{align*}

\end{itemize}

\medskip

The importance weights become
\begin{align}
\label{eq:weight_new1}
\omega_i^{(\ell)}
&=
\omega_{i-1}^{(\ell)}
\cdot
\frac{\mathcal{N}_i^{(\ell)}}{\mathcal{D}_i^{(\ell)}},
\end{align}
where
\begin{align*}
\mathcal{N}_i^{(\ell)}
&=
p(\mu_i^{(\ell)})\,
f(\boldsymbol{y}_i^{(\ell)}\mid\mu_i^{(\ell)})
\left(
\frac{1}{M}\sum_{m=1}^M
\kappa\big(d(y_{i,m}^{(\ell)},x);\epsilon_i\big)
\right)
\\
&\quad \times
L_{i-1}\big(
(\mu_i^{(\ell)},\boldsymbol{y}_i^{(\ell)}),
(\mu_{i-1}^{(\ell)},\boldsymbol{y}_{i-1}^{(\ell)})
\big),
\\[1em]
\mathcal{D}_i^{(\ell)}
&=
p(\mu_{i-1}^{(\ell)})\,
f(\boldsymbol{y}_{i-1}^{(\ell)}\mid\mu_{i-1}^{(\ell)})
\left(
\frac{1}{M}\sum_{m=1}^M
\kappa\big(d(y_{i-1,m}^{(\ell)},x);\epsilon_{i-1}\big)
\right)
\\
&\quad \times
K_i\big(
(\mu_{i-1}^{(\ell)},\boldsymbol{y}_{i-1}^{(\ell)}),
(\mu_i^{(\ell)},\boldsymbol{y}_i^{(\ell)})
\big).
\end{align*}

Since
\begin{align*}
\frac{
L_{i-1}\big((\mu_i^{(\ell)},\boldsymbol{y}_i^{(\ell)}),
(\mu_{i-1}^{(\ell)},\boldsymbol{y}_{i-1}^{(\ell)})\big)
}{
K_i\big((\mu_{i-1}^{(\ell)},\boldsymbol{y}_{i-1}^{(\ell)}),
(\mu_i^{(\ell)},\boldsymbol{y}_i^{(\ell)})\big)
}
&=
\frac{
\tilde{p}_{\epsilon_i}(\mu_{i-1}^{(\ell)},\boldsymbol{y}_{i-1}^{(\ell)}\mid x)
}{
\tilde{p}_{\epsilon_i}(\mu_i^{(\ell)},\boldsymbol{y}_i^{(\ell)}\mid x)
} \\
&=
\frac{
p(\mu_{i-1}^{(\ell)})
f(\boldsymbol{y}_{i-1}^{(\ell)}\mid\mu_{i-1}^{(\ell)})
\left(
\frac{1}{M}\sum_{m=1}^M
\kappa \big(d(y_{i-1,m}^{(\ell)},x); \epsilon_i\big)
\right)
}{
p(\mu_i^{(\ell)})
f(\boldsymbol{y}_i^{(\ell)}\mid\mu_i^{(\ell)})
\left(
\frac{1}{M}\sum_{m=1}^M
\kappa \big(d(y_{i,m}^{(\ell)},x); \epsilon_i\big)
\right)
},
\end{align*}
Eq.~\eqref{eq:weight_new1} simplifies to
\begin{equation}
\label{eq:weight_new2}
\omega_i^{(\ell)}
=
\omega_{i-1}^{(\ell)}
\cdot
\frac{
\sum_{m=1}^M 
\kappa \big(d(y_{i-1,m}^{(\ell)},x); \epsilon_i\big)
}{
\sum_{m=1}^M 
\kappa \big(d(y_{i-1,m}^{(\ell)},x); \epsilon_{i-1}\big)
}.
\end{equation}

\subsection{Proofs}
\label{prop:bias_proof}
\begin{proof}[Proof of Proposition~\ref{prop:bias}] 

\begin{align*}
    \beta(\mu;\epsilon, M) &:=\big\vert\tilde{f}_{\epsilon,M}(x\mid \mu) - f_{\epsilon,M}(x\mid \mu)\big\vert\\
    &= \frac{1}{M} \bigg\vert\sum \limits_{m=1}^M \kappa(d(y_m,x);\epsilon) - \sum\limits_{m=1}^M\mathbb{I}_{A_{\epsilon,x}}(y_m) \bigg\vert\\
    &=\frac{1}{M} \bigg\vert \sum \limits_{m=1}^M \mathbb{I}_{A_{\epsilon,x}}(y_m)+ \sum \limits_{m=1}^M\mathbb{I}_{A^c_{\epsilon,x}}(y_m) \exp\bigg[-\dfrac{d(x,y_m)}{\epsilon^2}\bigg]-\sum \limits_{m=1}^M\mathbb{I}_{A_{\epsilon,x}}(y_m)\bigg\vert\\
    &=\frac{1}{M}\sum \limits_{m=1}^M\mathbb{I}_{A^c_{\epsilon,x}}(y_m) \exp\bigg[-\dfrac{d(x,y_m)}{\epsilon^2}\bigg]\\
    &\leq \sum\limits_{m=1}^M \exp\bigg[-\dfrac{d(x,y_m)}{\epsilon^2}\bigg] 
\end{align*}
where  $A^c_{\epsilon, x} :=\{y \in \mathcal{X}^i : d(y, x) > \epsilon\}$.
\end{proof}

\subsection{LF-IBIS flowchart}
\label{appendix:LF-IBIS flowchart}
\begin{figure}[H]
    \centering
    \includegraphics[width=\textwidth]{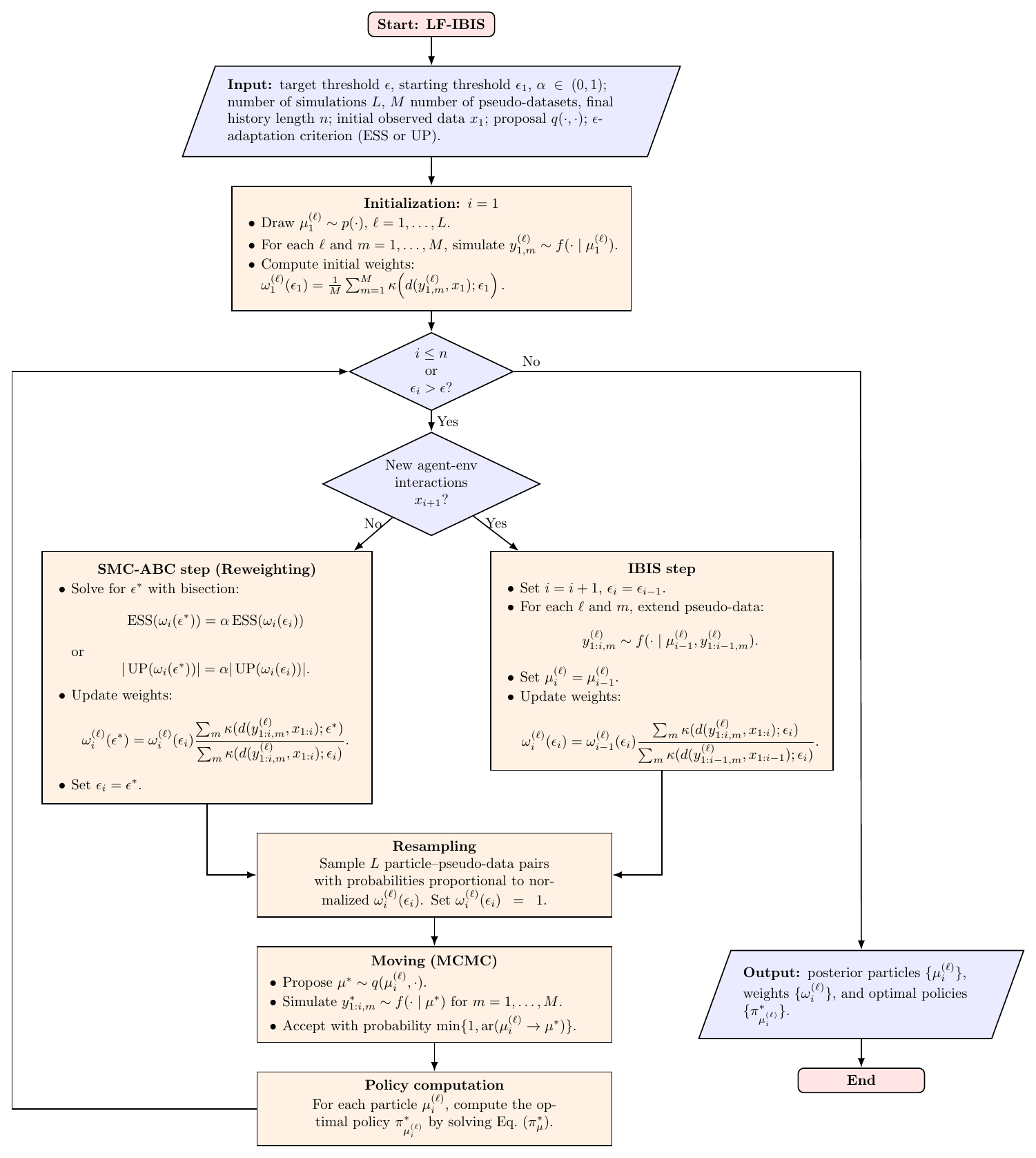}
   \caption{Overview of the proposed LF-IBIS algorithm. Starting from an initial approximation of the posterior distribution, the algorithm iteratively performs either an SMC-ABC step, in which the ABC tolerance is adaptively reduced, or an IBIS step, in which newly observed data are incorporated through likelihood-free importance reweighting. When necessary, particles are resampled and rejuvenated via an MCMC move. The resulting particles provide an approximation of the posterior distribution over model parameters and optimal policies.}
    \label{fig:lfibis-flowchart}
\end{figure}

\newpage
\section{Supplementary Material for Section~\ref{sec:experiments}}
\subsection{Beta proposal distribution} \label{sec:beta_proposal}
In the clinical trial experiment, at iteration $i$, each new particle is sampled according to the Beta proposal distribution $q(\mu_{0,i-1}^{(\ell)}, \cdot)$ defined as follow:
    \[
        \mu_{0,i}^{(\ell)} \sim q(\mu_{0,i-1}^{(\ell)}, \cdot) = \text{Beta}\big(a_{i}^{(\ell)}, b_{i}^{(\ell)}\big),
    \]
    where $a_{i}^{(\ell)}$ and $b_{i}^{(\ell)}$ are computed via the method of moments as
    \begin{align*}
        a_{i}^{(\ell)} &= \mu_{0,i-1}^{(\ell)} \left( \frac{\mu_{0,i-1}^{(\ell)} (1-\mu_{0,i-1}^{(\ell)})}{2 \widehat{\text{Var}}_{i-1}(\mu_0)} - 1 \right), \\
        b_{i}^{(\ell)} &= \big(1-\mu_{0,i-1}^{(\ell)}\big) \left( \frac{\mu_{0,i-1}^{(\ell)} (1-\mu_{0,i-1}^{(\ell)})}{2 \widehat{\text{Var}}_{i-1}(\mu_0)} - 1 \right).
    \end{align*}

    By construction, this proposal distribution satisfies
    \begin{align*}
        \mathbb{E}[\mu_{0,i}^{(\ell)}] &= \frac{a_{i}^{(\ell)}}{a_{i}^{(\ell)} + b_{i}^{(\ell)}} = \mu_{0,i-1}^{(\ell)}, \\        
        \text{Var}[\mu_{0,i}^{(\ell)}] &= \frac{a_{i}^{(\ell)} b_{i}^{(\ell)}}{(a_{i}^{(\ell)} + b_{i}^{(\ell)})^2 (a_{i}^{(\ell)} + b_{i}^{(\ell)} + 1)} = 2 \widehat{\text{Var}}_{i-1}(\mu_0) .
    \end{align*}
        
    Where the empirical variance of the particles at iteration $i-1$ is defined as
    \[
        \widehat{\text{Var}}_{i-1}(\mu_0) = \frac{1}{L-1}\sum_{\ell=1}^L \big(\mu_{0,i-1}^{(\ell)} - \overline{\mu}_{0,i-1}\big)^2,
    \]
    where $\overline{\mu}_{0,i-1} = \frac{1}{L}\sum_{\ell=1}^L \mu_{0,i-1}^{(\ell)}$.

    Let
    \[
    V_{i-1} \coloneqq 2\,\widehat{\operatorname{Var}}_{i-1}(\mu_0).
    \]
    When the empirical variance is too large, i.e., $V_{i-1} \ge \mu_{0,i-1}^{(\ell)}(1-\mu_{0,i-1}^{(\ell)})$, no unimodal Beta distribution exists with the desired moments. In this case, the proposal parameters are set such that $a_i^{(\ell)} + b_i^{(\ell)} = 1$, yielding a non-unimodal Beta distribution with substantial mass near the boundaries of the unit interval. 
    To avoid this behavior, we enforce the following condition for the method-of-moments parameters: for each particle \(\ell\),
    \[
    a_{i}^{(\ell)} =
    \begin{cases}
    \mu_{0,i-1}^{(\ell)}, & \text{if } V_{i-1} \ge \mu_{0,i-1}^{(\ell)}\big(1-\mu_{0,i-1}^{(\ell)}\big),\\[6pt]
    \mu_{0,i-1}^{(\ell)}\!\left(\dfrac{\mu_{0,i-1}^{(\ell)} \big(1-\mu_{0,i-1}^{(\ell)}\big)}{V_{i-1}} - 1\right), & \text{otherwise},
    \end{cases}
    \]
    \[
    b_{i}^{(\ell)} =
    \begin{cases}
    1-\mu_{0,i-1}^{(\ell)}, & \text{if } V_{i-1} \ge \mu_{0,i-1}^{(\ell)}\big(1-\mu_{0,i-1}^{(\ell)}\big),\\[6pt]
    (1-\mu_{0,i-1}^{(\ell)})\!\left(\dfrac{\mu_{0,i-1}^{(\ell)} \big(1-\mu_{0,i-1}^{(\ell)}\big)}{V_{i-1}} - 1\right), & \text{otherwise}.
    \end{cases}
    \]

In the logistic reparameterized setting of the clinical application, we do not face the constraint of keeping the variance small in order to avoid negative parameters, as previously required for the Beta proposal on $\mu_0$ and $\mu_1$. Since $\beta_0, \beta_1 \in \mathbb{R}$, the bivariate Normal proposal is always well defined, and scaling the covariance matrix does not introduce any feasibility issues.

\subsection{Additional Results: KL Divergence Analysis}\label{sec:supplementary2}

We compare the posterior distributions obtained in the experiments under the different inference strategies by computing the Kullback--Leibler (KL) divergence between the exact Bayesian posterior and its approximations produced by AR-ABC and by the two variants of LF-IBIS (ESS- and UP-based), for both experiments.
For each method, we report the mean and standard deviation of the KL divergence computed over ten independent runs of the algorithm, all performed using the same observed history as in the experiments above.
The KL divergence is reported only for the non-reparameterized versions of the experiments, since under the logistic reparameterization the exact Bayesian posterior is not available.
All KL divergences are estimated from posterior samples using kernel density approximations.

\begin{table}[H]
\centering
\footnotesize
\caption{Mean and standard deviation of the Kullback--Leibler divergence computed over ten independent runs of the algorithms with respect to the exact Bayesian posterior. Results are reported for inference on the parameters $\mu$ and on the policy $\pi$.}
\medskip
\label{tab:kl_comparison}
\begin{tabular}{lcccc}
\toprule
\textbf{Method} 
& \multicolumn{2}{c}{\textbf{Exp.~1 (Observation-based)}} 
& \multicolumn{2}{c}{\textbf{Exp.~2 (Utility-based)}} \\
\cmidrule(lr){2-3} \cmidrule(lr){4-5}
& $\mu$ & $\pi$ & $\mu$ & $\pi$ \\
\midrule
LF-IBIS (ESS) 
& 0.0954 (0.0132) & 0.0632 (0.0101) & 1.1589 (0.3567) & 0.3034 (0.2010) \\
LF-IBIS (UP) 
& 0.1239 (0.0211) & 0.0640 (0.0138) & 0.6992 (0.2246) & 0.1207 (0.0338) \\
AR-ABC 
& 0.0085 (0.0035) & 0.0052 (0.0027) & 1.1671 (0.0441) & 0.9696 (0.0431) \\
\bottomrule
\end{tabular}
\end{table}

\end{document}